\documentclass[pdflatex,sn-mathphys-num]{sn-jnl}

\usepackage{graphicx}%
\usepackage{multirow}%
\usepackage{amsmath,amssymb,amsfonts}%
\usepackage{amsthm}%
\usepackage{mathrsfs}%
\usepackage[title]{appendix}%
\usepackage{xcolor}%
\usepackage{textcomp}%
\usepackage{manyfoot}%
\usepackage{booktabs}%
\usepackage{algpseudocode}%
\usepackage{listings}%
\usepackage{float}
\usepackage{subfig}
\usepackage{comment}

\usepackage[ruled,vlined,linesnumbered]{algorithm2e}

\DeclareMathOperator*{\argmin}{argmin}

\newtheorem{assumption}{Assumption}
\newtheorem{theorem}{Theorem}
\newtheorem{lemma}[theorem]{Lemma}
\newtheorem{proposition}[theorem]{Proposition}
\newtheorem{corollary}[theorem]{Corollary}

\newtheorem{definition}{Definition}

\title{Working Principles of Model-Based Genetic Algorithms in PAC Framework: A Theory of Problem Decomposition}

\author*[1]{\fnm{Tian-Li} \sur{Yu}}\email{tianliyu@ntu.edu.tw}
\author[1]{\fnm{Chi-Hsien} \sur{Chang}}\email{d07921004@ntu.edu.tw}
\author[2]{\fnm{Ying-ping} \sur{Chen}}\email{ypchen@cs.nycu.edu.tw}

\affil*[1]{\orgdiv{Department of Electrical Engineering}, \orgname{National Taiwan University}, \country{Taiwan}}
\affil[2]{\orgdiv{Department of Computer Science}, \orgname{National Yang Ming Chiao Tung University}, \country{Taiwan}}

\usepackage[switch]{lineno}

\begin{document}

\newcommand{\TLY}[1]{{\color{violet}#1}}
\newcommand{\YP}[1]{{\color{blue}#1}}
\newcommand{\Salima}[1]{{\color{teal}#1}}

\newcommand{\note}[1]{{\color{red}#1}}

\abstract{
The concepts of linkage, building blocks, and problem decomposition have long existed in the genetic
algorithm field and have guided the development of model-based genetic algorithms for decades. However, their definitions are usually vague, making it difficult to develop theoretical support. This paper provides an algorithm-independent definition to describe the concept of linkage. With this definition, the paper proves that  any problem with a bounded degree of linkage is decomposable and that proper problem decomposition is possible via linkage learning. The way of decomposition given in this paper also offers a new perspective on nearly decomposable problems with bounded difficulty and building blocks from the theoretical aspect. Finally, this paper relates problem decomposition to probably approximately correct (PAC) learning and proves that the global optima of problems with bounded decomposition difficulty are PAC learnable and the decomposition is decidable in polynomial time under certain conditions.
}


\maketitle

\section{Introduction}

In the field of genetic algorithm (GA), the notion of linkage was first brought into this field by \mbox{Holland}~\cite{Holland:1975}. Before long, the meaning of linkage was split~\cite{YPC:ll}. Some have used it to indicate the degree of genes recombined together during the recombination operator; some have used it to indicate whether genes are separated when the problem is decomposed into subproblems. The former views linkage from a more operational aspect, while the latter views linkage from a more problem-difficulty aspect. To avoid confusion, the term ``epistasis'' was borrowed from biology to focus the research on problem difficulty~\cite{Rochet:1997:epistasis}. In 1992, \mbox{Goldberg} \textit{et al.} proposed a GA design theory~\cite{Goldberg:1992:ga-decomposition}, later widely referred to as the \textit{building-block hypothesis}, which pointed out that the development of GAs should focus on problem decomposition and that the atomic fragments of chromosomes are called building blocks. 

Despite the vague meanings of linkage and building blocks~\cite{Holland:1975, Goldberg:1992:ga-decomposition}, efforts have been put into developing new GAs~\cite{Goldberg:doi:2002} to conquer the difficulty of deception~\cite{Goldberg:1987:trap}. Among different attempts, a branch called the model-based GA (MBGA) has reached notable achievements~\cite{Larranaga:2002:eda,Pelikan:2005:pmbga,Pelikan:eda:2015}. The key idea is to build a model based on linkage detection and then use the model to guide recombination. Some researchers have used the term model-building GA, which emphasizes the model building phase. The term model-based GA includes all model-building GAs and is more general. For example, in the field of gray-box optimization~\cite{Whitley:2019:gray}, the model is usually provided as a prior, and hence the research belongs to model-based GAs, but may not be considered as model-building GAs. Even though the meanings of these two terms are different, such a difference does not affect the context of this paper.

Due to the success in application, researchers have put efforts into developing theoretical supports for MBGAs. Here we list some major research as background. The works in \cite{Pelikan:2003:boa-scalibility, Yu:2007:pop-sizing} focused on the population sizing for model building. The work in \cite{Fan:2013:nfe} discussed building blocks from a practical point of view and empirically showed that the optimal problem decomposition, defined by minimum number of function evaluations, of the concatenated trap is indeed composed of subfunctions. The work in \cite{Yu:2004:model-q-e} analyzed the effect of imperfect model building. The works in \cite{Thirens:2011:om, Thierens:2012:evolvability} focused on analyzing optimal mixing, a specific operator used in state-of-the-art MBGAs, among other operators. The works in \cite{Przewozniczek:2020:ell, Przewozniczek:2021:dld} went in another direction by discussing what linkage is and categorizing linkage learning into statistical and empirical ones.

All the aforementioned research, among others, has made significant progress in understanding MBGAs and helping the development of modern ones. 
However, because of the diverse and stochastic nature of GAs, the generality and rigorousness of algorithm-dependent analyses are usually limited. The purpose of this paper is to develop a formal mathematical framework to explain the working principles of MBGAs. To achieve generality, the proposed framework focuses on the relationship between linkage and problem decomposition. The major contributions of this paper come as follows:

\begin{enumerate}
	\item We propose a mathematical framework to investigate problem decomposition.
	\item With the framework, we are able to obtain the following results.
	\begin{itemize}
		\item The problem decomposition theorem states that the minimum decomposition directly relates to binary epistatic relations on genes. 
		\item The epistasis blanket theorem states that the alleles of genes can be correctly set 		
		as long as their epistatically related genes within a specific range are correctly set. 
		\item The global optima of nearly decomposable problems with bounded difficulty are probably approximately correct (PAC) learnable~\cite{Valiant:pac:1984}, and such decomposition is PAC-decidable under certain conditions.
	\end{itemize}
	\item We relate the framework to the working principles of MBGAs.
\end{enumerate}

To the best of our knowledge, we are the first to propose a mathematical framework that provides a solid foundation for obtaining all the aforementioned results taking into consideration the concept of linkage, problem decomposition, and model building. 

The remainder of this paper is organized as follows. Section~\ref{sec:EpistasisProperties} defines epistasis and shows its properties. Section~\ref{sec:eg} constructs the epistasis graph and theoretically relates it to problem decomposition. Section~\ref{sec:learnability} proves the learnability of the global optimum and the decomposition under the PAC framework. Finally, Section~\ref{sec:conc} concludes this paper.

\section{Epistasis and Its Properties}
\label{sec:EpistasisProperties}

\noindent In this section, we formally define epistasis as used in this paper. Frequently used operators and symbols are also introduced. Using these symbols, we demonstrate several important properties of epistasis with commonly used test problems in the GA field, which are repeatedly referenced in the remainder of this paper.

\begin{table}
	\centering
	\caption{Fitness of the exemplified test problems. The latter four problems consist of $m$ subproblems, each of which consists of 4 genes.}
	\label{tbl:problems}
	\renewcommand{\arraystretch}{2.0} 
	
	\begin{itemize}
		\item $u_4 = \sum_{j=0}^3 b_j$
		\item $trap_4(b_0, \ldots, b_{3}) = 
		\begin{cases}
			4 & \text{if } u_4 = 4 \\
			3-u_4 & \text{otherwise} 
		\end{cases}$
		\item $niah_4(b_0, \ldots, b_3) = 
		\begin{cases}
			4 & \text{if } u_4 = 4 \\
			0 & \text{otherwise}
		\end{cases}	$
		\item $L_{i} =
		\begin{cases}
			1 & \text{if } i = 0\\
			1 & \text{if } i \geq 1\text{ and } L_{i-1} = 1 \text{ and } \\
			& trap_4(y_{4(i-1)}, \ldots, y_{4(i-1)+3}) = 4\\
			0 & \text{otherwise}
		\end{cases}$
	\end{itemize}
	
	\begin{tabular}{|p{2.2cm}|c|}
	\hline
	\textbf{Name} & \textbf{Fitness Function}\\ 
	\hline
	\hline
	\textsc{OneMax} &
	\( f(\vec{y}) = \sum_{i=0}^{\ell-1} y_i \) \\
	\hline
	\textsc{LeadingOnes} &	
	\( f(\vec{y}) = \sum_{i=0}^{\ell-1} \prod_{j=0}^i y_j \) \\
	\hline
	\textsc{CTrap} & 	
	\(
	\begin{aligned}
	& f(\vec{y}) = \sum_{i=0}^{m-1} trap_4(y_{4i}, \ldots, y_{4i+3})\\	
	\end{aligned}
	\) \\
	\hline
	\textsc{CycTrap} & 	
	\(
	\begin{aligned}
	& f(\vec{y}) = \sum_{i=0}^{m-1} trap_4(y_{3i}, \ldots, y_{3i\oplus3}), \\
	& \text{where } a \oplus b = (a+b) \mod \ell\\
	\end{aligned}
	\) \\
	\hline
	\textsc{CNiah} &	
	\(\begin{aligned}
	&f(\vec{y}) = \sum_{i=0}^{m-1} niah_4(y_{4i}, \ldots, y_{4i+3})\\
	\end{aligned}
	\)\\
	\hline
	\textsc{LeadingTraps} & 
	\(
	\begin{aligned}
	&f(\vec{y}) =  \sum_{i=0}^{m-1} L_{i} \cdot trap_4(y_{4i}, \ldots, y_{4i+3})\\
	\end{aligned}
	\) \\
	\hline		
	\end{tabular}
\end{table}

\subsection{Exemplified Test Problems}
This subsection defines several test problems for discussion purposes. Specifically, these problems are the one-max (\textsc{OneMax}), the leading-ones (\textsc{LeadingOnes}), the concatenated trap (\textsc{CTrap}), the cyclic trap (\textsc{CycTrap}), the concatenated needle-in-a-haystack (\textsc{CNiah}), and the leading-traps (\textsc{LeadingTraps}) problems. To eliminate the correlation between the decomposition and locus, all fitness functions are defined over $\vec y = (y_i)$, which is a permutation of the chromosome $\vec x = (x_i)$. Their formulae are listed in Table~\ref{tbl:problems}.

The reason for the choice is as follows. \textsc{OneMax} has long been used for theoretical analysis in the GA field. \textsc{LeadingOnes}, \textsc{CTrap}, and \textsc{CycTrap} have been used as benchmark problems to test the ability/performance of GAs~\cite{Thierens:OMEA:2011, Hsu:dsmga2:2015}.  \textsc{CNiah} is used to demonstrate different types of epistasis through comparison with \textsc{CTrap}. We additionally define \textsc{LeadingTraps}, which follows the concept of \textsc{BlockLeadingOnes} in~\cite{Doerr:2023:blockleadingones}, to demonstrate the integration of these concepts through problem decomposition in optimizing complex problems.

The commonly used benchmark NK-landscape embeds an overlapping cyclic structure that is similar to \textsc{CycTrap} but is not adopted in this paper due to its randomness. Other commonly used benchmark problems, such as the spin-glasses, \textsc{MaxSat}, \textsc{MaxCut}~\cite{Thierens:OMEA:2011, Hsu:dsmga2:2015} are excluded for the same reason. One may, of course, apply the mathematical tools proposed in this paper to any particular instance of those problems, but the results are not guaranteed to be general due to the randomness of structures. Note that given any fitness function, epistatic relations can always be constructed by definition. The selection of exemplified problems does not limit the applicability of our theory. As an example, the epistatic relations of a manually designed instance of \textsc{MaxSat} are illustrated in the next section.

\subsection{Basic Operations}
Here we make assumptions and define symbols. 
Important symbols used in this paper are summarized in Table~\ref{tbl:sym}. To start with, let the problem size (or equivalently the length of chromosomes) be $\ell$ and define $V$, where $|V|=\ell$, as the set of loci, indexed starting from 0. $f$ is the fitness function, assumed to be maximization without loss of generality. $g$ is the indexable vector of the unique global optimum (Assumption~\ref{asm:unique}) under fitness $f$:  $\forall v\in V,\; g[v] \in \{0,1\}$. The allele of the global optimum at locus $v$ is $g[v]$ with its complement being $\overline{g[v]}$. 

\begin{table}
	\caption{Important symbols and their meanings used in this paper.}
	\label{tbl:sym}	
	\begin{tabular}{p{0.15\columnwidth}p{0.75\columnwidth}}
		\toprule	
		$f$ & The fitness function, assumed to be maximization.\\		
		$V$ & The set of all loci $\{0, \ldots, \ell-1\}$, $|V|=\ell$. (later used in graph).\\	
		$g$ & The indexable vector of the unique global optimal pattern under fitness $f$.\\ 
		    & $\forall v\in V,\; g[v] \in \{0,1\}$.\\		
		$A$ & Assignment that is indexable with the symbol `*' indicating no assignment at that locus.\\
		    & $\forall v\in V,\; A[v] \in \{0, 1, *\}$. \\
		$\mathcal{C}(A)$ & The coverage of an assignment $A$: The set of loci where the alleles are assigned in $A$.\\
		$\Psi_{A}[v]$ & The set of constrained optima at locus $v$ with the constraint of assignment $A$.\\  
		    &  $\forall v\in V$, $\Psi_{A}[v] \in \{ \{0\} , \{1\}, \{0,1\} \}$. \\	
		$S \Rightarrow v$ & $S$ is $|S|$-epistatic to $v$.\\
		$u\rightarrow v$ & $u$ is strictly epistatic to $v$.\\
		$u\dashrightarrow v$ & $u$ is non-strictly epistatic to $v$.\\
		$\mathcal{IN}^i(v)$ & The set of vertices that connect to $v$ with $i$ edges. $\mathcal{IN}^*(v)$ is the set of vertices that connect to $v$.\\
		$\mathcal{M}_{SO}(v)$ & The minimum stationary optimum of locus $v$.\\
		$r$ & The minimum deception rate of a problem.\\
		\bottomrule
	\end{tabular}
	
\end{table}

\begin{assumption}
	\label{asm:unique}
	The global optimum is unique.
\end{assumption}

Assumption~\ref{asm:unique} provides properties for the ease of derivations while limits the applicability of the conclusion of this paper on the other hand. Nevertheless, when niching techniques~\cite{Goldberg:1989:ga-book} are applied to problems containing multiple global optima and result in several clusters, where each cluster evolves toward each global optimum respectively, our results also apply to each cluster.

We start our derivations with the following definitions of full/partial gene assignments and their coverages. 

\begin{definition} Assignments and coverage	
	\begin{itemize}
		\item A \textbf{gene assignment} is a locus-allele pair $(v, a)$, where $v \in V$ and $a\in \{0, 1\}$. 
		\item For all $v\in V$, the assignment $(v, g[v])$ is shortened to $(v, g)$ and is called the \textbf{correct} assignment; likewise, the assignment $(v, \overline{g[v]})$ is shortened to $(v, \overline{g})$ and is called the \textbf{incorrect} assignment.
		\item 	An \textbf{assignment}, denoted by $A$, is a set of gene assignments. Multiple assignments for the same gene are not allowed. An assignment is a \textbf{full} assignment if every gene is assigned; otherwise, it is a \textbf{partial} assignment. 
		For convenience, we make the assignment $A$ indexable:
		\begin{linenomath*}
		\begin{equation*}
		 A[v] = \begin{cases} a, & \text{ if } (v,a) \in A\\ *,& \text{ otherwise. }\end{cases}
		\end{equation*}
		\end{linenomath*}		 
		\item 	A \textbf{batch assignment}: For convenience, we overload the symbol: $\{(S, a)\}$, where $S \subseteq V$ and $a \in \{ 0, 1\}$, to be equivalent to $\bigcup_{s\in S} \{(s, a)\}$. In addition, we often need to assign a batch of genes to be either the corresponding global optimal pattern or its complement. We denote 
		$\bigcup_{s\in S} \{(s, g[s])\}$ by $\{(S, g)\}$ and $\bigcup_{s\in S} \{(s, \overline{g[s]})\}$ by $\{(S, \overline{g})\}$.
		\item  	The \textbf{coverage} of an assignment $A$ is a set of loci whose alleles are assigned by $A$, denoted by $\mathcal{C}(A) = \{v \;|\; A[v] \neq *\}$.  For a full assignment $A$, $|\mathcal{C}(A)| = \ell$.
		\item \textbf{Applying} assignment $A$ to chromosome $\vec y$ results in $\vec y^A$, where the alleles at loci specified in $A$ are overridden.
	\end{itemize}	
\end{definition}

The above definition of assignment resembles Holland's notation of schema and messy encoding~\cite{Holland:1975}, from which we did borrow, but the meaning of the symbol `*' is different. Here are examples of assignments. Let $\ell=4$ and $A = {\{(1,0), (3,1)\}}$. Then we have $A[0] = *$, $A[1] = 0$, $A[2] = *$, $A[3] = 1$, and $\mathcal{C}(A) = \{1,3\}$. One may also use the batch assignment to assign alleles. For example, $A = \{(\{1, 3\}, 1)\}$ indicates that $A[1] = A[3] = 1$ and $A[0] = A[2] = *$. Next, we define the constrained optima. 

\begin{definition}
	The \textbf{constrained optima} with the constraint of assignment $A$, denoted by $\Psi_{A}$, is a set of chromosomes with assignment $A$ applied and with maximum fitness: 
	For all $\psi \in \Psi_{A}$ and for all chromosome $\vec y$, we have $f(\psi) \geq f(\vec y^A)$. Let $\Psi_{A}[v]$ be the set of alleles at locus $v$ of the constrained optima. 
	$\forall v\in V$, $\Psi_{A}[v] \in \{ \{0\}, \{1\}, \{0,1\}\}$.
\end{definition}

Note that, even though the global optimum is unique, the constrained optima may not be, and hence are denoted in the form of sets. For example, for a 4-bit \textsc{CNiah} (only one subproblem), the constrained optima with the constraint of assignment $\{(0,0), (1,0)\}$ is $\Psi_{\{(0,0), (1,0)\}} = \{0000, 0001, 0010, 0011\}$, while the global optimum $1111$ is unique. In this example, $\Psi_{\{(0,0),(1,0)\}}[0] = \{0\}$ and $\Psi_{\{(0,0),(1,0)\}}[2] = \{0,1\}$.

\begin{definition} \textbf{Evaluating an assignment}	
	Naturally, only full assignments are evaluable. For convenience, we make a partial assignment $A$ also evaluable by evaluating the constrained optima given the constraint of assignment $A$. If there exist multiple optima given the constraint, the evaluation lies in either one of them since their fitness values are equal. Such evaluation is denoted by $f(A)$. Mathematically, $f(A) = f(\psi)$, where $\psi \in  \Psi_A$.
\end{definition}

If an assignment consists solely of a fragment of the global optimum, the constrained optimum with the assignment is the global optimum itself (Proposition~\ref{prop:opt}). Next, a constrained optimum remains unchanged with additional constraints if the additional constraints come from the constrained optimum itself (Proposition~\ref{prop:opt-remain1}). 
Their proofs are straightforward and are omitted here. Finally, Proposition~\ref{prop:opt-remain1} can be generalized to Proposition~\ref{prop:opt-remain}.

\begin{proposition}
	\label{prop:opt}
	$\forall S \subseteq V, \forall v\in V$, we have $ \Psi_{\{(S, g)\}}[v]  = \{ g[v]\}$.
\end{proposition}

\begin{proposition}
	$\forall A$ and $\forall v\in V$, $\Psi_A[v] = \{a\}$ implies $\Psi_{A \cup \{(v,a)\}} = \Psi_A$.\label{prop:opt-remain1}
\end{proposition}

\begin{proposition}
	$\forall A$ and $\forall v\in V$, $a \in \Psi_A[v]$ implies $\Psi_{A \cup \{(v,a)\}} \subseteq \Psi_A$.\label{prop:opt-remain}
	\begin{proof}
	For the case where $\Psi_{A}[v]$ = $\{a\}$, the statement holds by Proposition~\ref{prop:opt-remain1}.\\
	For the case where $\Psi_{A}[v]$ = ${\{a, \overline{a}\}}$, the additional constraint ${\{(v, a)\}}$ merely removes the patterns in $\Psi_{A \cup \{(v, \overline{a})\}}$ from $\Psi_{A}$, \textit{i.e.}, $\Psi_{A \cup \{(v, a)\}}$ = $\Psi_{A}-$   $\Psi_{A \cup \{(v, \overline{a})\}}$.  Hence, $\Psi_{A \cup \{(v, a)\}} \subseteq \Psi_{A}$.
	\end{proof}
\end{proposition}

\subsection{Epistasis}
To investigate the properties of linkage, here we give the definition of epistasis used in this paper with previously defined mathematical notation.

\begin{definition}  Epistasis $S \Rightarrow v$\\
	Suppose that $S \subseteq V-\{v\}$ and $v \in V$. 
	$S$ is an \textbf{$|S|$-epistasis} to $v$, or $S$ is \textbf{$|S|$-epistatic} to $v$, if and only if $\forall s \in S$, there exists an assignment $A$ with $\mathcal{C}(A)=S$, such that $\Psi_A[v] \neq \Psi_{A-\{(s,a)\}}[v]$, where $(s,a) \in A$. $|S|$ is the \textbf{order} of the epistasis. Such a relation is denoted by $S \Rightarrow v$.

	We further categorize \textbf{order-1} epistases, $\{u\} \Rightarrow v$, into strict and non-strict ones:
	\begin{itemize}
	\item $u\rightarrow v$\\
	This is the case for $\Psi_{\{(u, g)\}}[v] \cap \Psi_{\{(u, \overline{g})\}}[v] = \phi$. In other words, $\Psi_{\{(u, \overline{g})\}}[v] = \{\overline{g[v]}\}$ (note that $\Psi_{\{(u, g)\}}[v] = \{g[v]\}$ by Proposition~\ref{prop:opt}). We say that $u$ is \textbf{strictly} epistatic to $v$.

	\item $u\dashrightarrow v$\\
	This is the case for 
	$\Psi_{\{(u, g)\}}[v] \cap \Psi_{\{(u, \overline{g})\}}[v] \neq \phi$. In other words, $\Psi_{\{(u, \overline{g})\}}[v] = \{g[v], \overline{g[v]}\}$. We say that $u$ is \textbf{non-strictly} epistatic to $v$.
	\end{itemize}

\end{definition}

Note that in the above definition, an empty set $\phi$ is not a 0-epistasis to any gene since it does not alter the optimal pattern at any locus. In short, $\forall v \in V,\; \phi \not \Rightarrow v $. In an 8-bit \textsc{CTrap}, where [0-1-2-3] and [4-5-6-7] are two subproblems, $\{4,5,6\}$ is 3-epistatic to $7$. Note that the definition represents a many-to-one relation. One might think that a many-to-many relation would be more general. Nevertheless, stating that $\{a, b\}$ is epistatic to $\{c, d\}$ is semantically equivalent to stating that $\{a, b\}$ is epistatic to $c$ and $\{a, b\}$ is epistatic to $d$.

There are many different possible ways to define epistasis to describe the concept of linkage, and each has its own benefits and drawbacks. Our definition is based on optimal patterns instead of based on a population of chromosomes with high fitness. One of the benefits is that such a definition is not related to any algorithms but to the problem itself, while the drawback is that we lose the possible benefits from the law of large numbers. 

\begin{table}[t]

\centering
	
\caption{An example of epistasis where the assignment is not complementary to the global optimum. In this example, $\{a,b\} \Rightarrow c$ since $\Psi_{\{(a,0),(b,1)\}}[c] = \{0\}$, while $\Psi_{\{(a,0)\}}[c] = \{1\}$ and $\Psi_{\{(b,1)\}}[c] = \{1\}$.}\label{tbl:top5}

\begin{tabular}{p{4.3cm}cp{4.3cm}}
&	
\begin{tabular}{cccc}
\toprule
\textbf{a} & \textbf{b} & \textbf{c} & \textbf{Fitness} \\ 
\midrule
1 & 1 & 1 & 10 \\
0 & 0 & 1 & 9  \\
1 & 0 & 1 & 8  \\
0 & 1 & 0 & 7  \\
0 & 1 & 1 & 6  \\
\multicolumn{3}{c}{others} & $<6$ \\ 
\botrule
\end{tabular}
& \\
\end{tabular}
\end{table}

Also, the definition does not require the assignments to be complementary to the global optimum, \textit{i.e.}, for $\overline{g[v]} \in \Psi_A[v]$, $A$ is not necessarily $\{( \mathcal{C}(A), \overline{g})\}$. The benefit is that such a definition is more general. Table~\ref{tbl:top5} illustrates an example of a 3-bit problem where the top-5 best chromosomes are shown. In the example, $\{a, b\}$ is epistatic to $c$ due to the assignment $a$ to 0 and $b$ to 1, while setting both $a$ and $b$ to 0 does not cause the optimal assignment at $c$ to be 0. Finally, our definition requires each locus in $S$ to be either essential to cause $v$ incorrectly assigned to mislead the algorithm or essential to cause $v$ correctly assigned such that the global optimum can be constructed. In other words, we disallow \textit{hitchhikers} in the epistasis: $S \Rightarrow v$ does not imply that  $\exists T \neq \phi, (S\cup T) \Rightarrow v$.

Next, we define strong and weak epistases as follows.
\begin{definition}Strong and weak epistases\\
	Given $S \Rightarrow v$, where $|S| \geq 2$, we call the epistasis \textbf{strong} if $\forall \text{nonempty }T\subset S$, $T \Rightarrow v$. The epistasis is \textbf{weak} if $\,\forall T\subset S$, $T \not\Rightarrow v$. For convenience, we define every 1-epistasis as strong.
\end{definition}
\noindent Note that the above definition is not a dichotomy. Strong and weak epistases merely represent two extreme cases. Most epistases are neither strong nor weak. In $ctrap_4$, the epistasis  $\{0,1,2\} \Rightarrow 3$ is strong since $\{0\} \Rightarrow 3$, $\{1\} \Rightarrow 3$, $\ldots$, $\{0,1\} \Rightarrow 3$, $\{0,2\} \Rightarrow 3$, and so on. In the example of Table~\ref{tbl:top5}, the epistasis $\{a,b\} \Rightarrow c$ is weak, because $\{a\} \not\Rightarrow c$ and $\{b\} \not\Rightarrow c$.

Intuitively, strong epistases easily mislead the direction of optimization, but are also easily detectable. On the other hand, weak epistases may not greatly affect the optimization, and observing such epistases resembles finding the needle in a haystack. In general, a weak epistasis $S \Rightarrow v$ can be observed only when certain patterns of size $|S|$ and the corresponding patterns deviating from each pattern with one bit exist in the population. In addition, such an observation is easily disrupted. If any of the alleles in $S$ is changed by crossover or mutation, the relation is not observable again. The argument here resembles the schema disruption by recombination in Holland's schema theorem~\cite{Holland:1975}. Therefore, the probability of observing a specific weak epistasis negatively relates to its order and generation and positively relates to the size of the initial population.

The following problem exemplifies the above argument of the observability of weak epistases. Consider a slightly modified \textsc{OneMax} problem:
\begin{linenomath*}
\begin{equation*}
\textsc{OneMax}'(\vec x)  = \begin{cases}
	1.5 & \text{if all genes are 0}\\
	\textsc{OneMax}(\vec x) & \text{otherwise}\\
\end{cases}
\end{equation*}	
\end{linenomath*}
In the 100-bit \textsc{OneMax}$'$, $\{0,1,\ldots, 98 \} \Rightarrow 99$ is a weak $99$-epistasis. Such an epistasis is only observable when all genes are 0. Even if it exists in the population, most crossover or mutation would disrupt the pattern and make it unobservable. 

\begin{figure}
\centering
\includegraphics[width=0.8\columnwidth, trim={18mm 0mm 26mm 7mm}, clip]{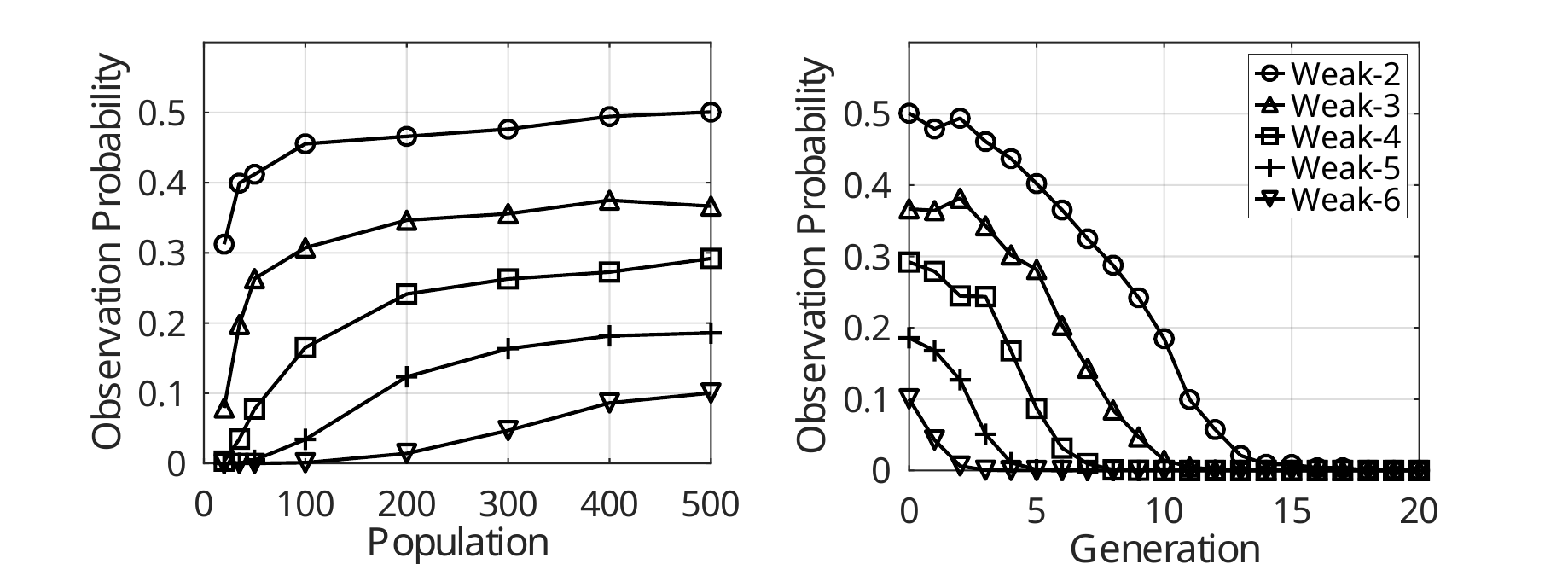}
\caption{Probability of observing a specific weak epistasis in different initial populations and in different generations. The left is at generation 0 (initial population); and the right is with a population of size 500.}
\label{fig:weak}
\end{figure}

As a result, we argue that weak epistases do not affect the evolutionary process by much since they are not easily observable. To make our point, we conduct the following experiment. The problem is of 25 bits with fitness defined as $\textsc{OneMax}'(x_0,\ldots, x_2)$ + $\textsc{OneMax}'(x_3,\ldots, x_6)$ + $\textsc{OneMax}'(x_7,\ldots, x_{11})$ + $\textsc{OneMax}'(x_{12},\ldots, x_{17})$ + 
$\textsc{OneMax}'(x_{18},\ldots, x_{24})$. The first block consists of three weak epistases of order 2 ($\{0,1\} \Rightarrow 2$, $\{0,2\} \Rightarrow 1$, and $\{1,2\} \Rightarrow 0$), and likewise the last block consists of seven weak epistases of order 6. We then perform a simple GA on the problem with binary tournament selection, uniform crossover with probability 0.9, and bit-wise mutation with  probability 0.01~\cite{Goldberg:1989:ga-book}. To obtain stable results, all probabilities are averaged over 1000 independent runs. As shown on the left of Figure~\ref{fig:weak}, the higher order of an epistasis is, the lower the probability is for it to be observed in the initial population. Also, such a probability increases with the  population size. When the population is sufficiently large, the probability converges to a stable value, but still not 1 due to cross competition with other blocks. Such a probability drops rapidly with GA operators over generation as shown on the right of Figure~\ref{fig:weak}.

Here, we assume that the problem-to-discuss does not contain weak epistasis (Assumption~\ref{asm:no-weak}) for the rest of this paper. With the assumption, any composition of non-epistatic relations remains non-epistatic (Proposition~\ref{prop:composition2}).

\begin{assumption}
The target problem does not contain weak epistasis.\label{asm:no-weak}
\end{assumption}

\begin{proposition}
	If every 2-epistasis to $v$ is non-weak, {\normalfont (}$ \{a\} \not\Rightarrow v $ \normalfont{and} $\{b\} \not\Rightarrow  v${\normalfont )} implies $\{a,b\} \not\Rightarrow v$.\label{prop:composition}	
	\begin{proof}
		Since every 2-epistasis to $v$ is non-weak,\\
		$\{a,b\} \Rightarrow v$ implies ($\{a\} \Rightarrow v$ or $\{b\} \Rightarrow v$), which is the contrapositive of the statement.
	\end{proof}
\end{proposition}

\begin{proposition}
Suppose that every epistasis to $v$ is non-weak. $S \Rightarrow v$ implies that $\exists s \in S$, $\{s\} \Rightarrow v$.\label{prop:at-least-one}
\begin{proof}
	$S\Rightarrow v$ is non-weak, so $\exists S'\subset S, S'\Rightarrow v$. Again, $S'\Rightarrow v$ is non-weak, so $\exists S''\subset S', S''\Rightarrow v$. Eventually, we  have some $s\in S$ such that $\{s\}\Rightarrow v$.
\end{proof}	
\end{proposition}

\begin{proposition}
	Suppose that every epistasis to $v$ is non-weak. Given $S \subseteq V$ where $\forall s\in S,\; \{s\} \not\Rightarrow v$,  we have $\forall S' \subseteq S,\; S' \not\Rightarrow v$.\label{prop:composition2}
	\begin{proof}
	Assume there exists $S'\subseteq S,\; S' \Rightarrow v$. By Proposition~\ref{prop:at-least-one}, $\exists s\in S'$, $\{s\} \Rightarrow v$. Since $s \in S$, that leads to a contradiction.
	\end{proof}
\end{proposition}

The next two propositions state that the existence of epistasis is essential to make a constrained optimum differ from the global optimum.

\begin{proposition}
	$\forall A, \forall v\not\in\mathcal{C}(A), \; \overline{g[v]} \in \Psi_A [v]$ implies $\exists S \subseteq \mathcal{C}(A), S \Rightarrow v$. Furthermore, if no weak epistasis exists in the problem, the above condition implies $\exists s \in \mathcal{C}(A), \{s\} \Rightarrow v.$\label{prop:someone}
\begin{proof}
Let $A_{\min}$ be the minimum subset of $A$ such that $\overline{g[v]} \in \Psi_{A_{\min}} [v]$. $A_{\min}$ is nonempty because $\Psi_{\phi} [v] = \{g[v]\}$. Since $A_{\min}$ is minimum, dropping any one more gene assignment makes the constrained optima different: $\forall (s,a) \in A_{\min}$, $\Psi_{A_{\min}} [v] \neq \Psi_{A_{\min} -\{(s,a)\}} [v]$. This fits the definition of epistasis, and hence we have $\mathcal{C}(A_{\min}) \Rightarrow v$, which makes the first claim of the proposition. The second claim is straightforward following Proposition~\ref{prop:at-least-one}.
\end{proof}	
\end{proposition}

\begin{proposition}
	For all assignments $A$ and $R$, $\Psi_{A \cup R} \neq \Psi_{R} $ implies $\exists (s,a)\in A, \overline{a} \in \Psi_{R}[s]$.\label{prop:alter}
\begin{proof}
	Assume the conclusion is not true: $\forall (s,a)\in A, \Psi_{R}[s] = \{a\}$. By Proposition~\ref{prop:opt-remain1},  $\Psi_{A \cup R} = \Psi_{R} $, which leads to a contradiction.
\end{proof}	
\end{proposition}

In general, deciding all epistatic relations costs $2^\ell$ function evaluations (full enumeration of the search space). If the global optimum is known in advance, a specific $k$-epistatic relation can be decided in $\Theta(2^k)$ evaluations. Capturing all $k$-epistases costs $\Theta(C^\ell_k)$ evaluations since all combinations of $k$ loci need to be considered. If $k$ is not known beforehand, capturing all epistases requires $\Theta(C^\ell_2 + C^\ell_3 + \ldots + C^\ell_\ell) = \Theta(2^\ell)$ function evaluations.  In other words, knowing the global optimum in advance does not reduce the number of evaluations needed. However, most modern MBGAs start with detecting low-order linkage and use that information to build up high-order linkage to reduce evaluations needed for model building. Here we use linkage as a general concept since different MBGAs detect different things, which also differ from the epistasis defined in this paper. 

Proposition~\ref{prop:at-least-one} states that any non-weak epistasis is comprised of at least one order-1 epistasis, which gives the theoretical support for building high-order epistasis from low-order ones. Such a property enables us to investigate the epistatic relations on the level of order-1, which results in the following definition of epistatic graph.
	
\begin{definition} Epistatic graph (EG) and in-closure
	\begin{itemize}
		\item $G = (V,E)$ is called an \textbf{epistatic graph}, where $E =\left\{ (u,v) \mid \{u\} \Rightarrow v\right\}$. \\
		Illustratively, the edge is solid if $u\rightarrow v$; it is dashed if $u\dashrightarrow v$.
		\item $\mathcal{IN}(v) = \{u \mid \{u\} \Rightarrow v \}$; \; $\mathcal{IN}^i(v) = \{u \mid \exists w \in \mathcal{IN}^{i-1}(v) \text{ s.t. } \{u\} \Rightarrow w \}$ for $i > 1$, where $\mathcal{IN}^1(v) = \mathcal{IN}(v)$. For convenience, define $\mathcal{IN}^0(v) = \{v\}$ and overload $\mathcal{IN}(S) = \bigcup_{v\in S}\;\mathcal{IN}(v)$ for $S\subseteq V$.
		\item $\mathcal{IN}^*(v)= \bigcup_{i=0}^{\ell-1} \mathcal{IN}^i(v)$ is the in-closure of $v$. 		
	\end{itemize}	
\end{definition}

Figure~\ref{fig:EG} illustrates the EGs of the six exemplified test problems. The EG of \textsc{OneMax} contains no edges since no epistasis exists between any pair of loci. The EG of \textsc{LeadingOnes} is a directed acyclic graph (DAG) with the adjacency matrix being entirely upper triangular, where a locus is non-strictly epistatic to each latter one. The EGs of \textsc{CTrap} and \textsc{CNiah} are quite similar. They both consist of $m$ disjoint maximal cliques, each corresponding to a subfunction in the fitness function. The only difference is that the epistases in \textsc{CTrap} are strict, while those in \textsc{CNiah} are non-strict; that indicates different degrees of deception of the subfunctions. The EG of \textsc{CycTrap} is interesting. Non-overlapping loci in the same block are epistatic to each other, but not epistatic to the overlapping ones. An overlapping locus is epistatic to non-overlapping loci in its adjacent blocks. Note that while the fitness function seems non-separable, the in-closures are separable. We will discuss this later. Finally, the EG of \textsc{LeadingTraps} consists of both strict and non-strict epistases. Loci in the same block form a maximal clique with strict epistases, while maximal cliques are non-strictly epistatic to others in a similar manner to that in \textsc{LeadingOnes}. The big dashed arrows between blocks in Figure~\ref{fig:EG} mean that each locus in the block is non-strictly epistatic to each locus in the next block.

\begin{figure}
\centering
\subfloat[\textsc{OneMax}]{
	\includegraphics[width=0.22\textwidth]{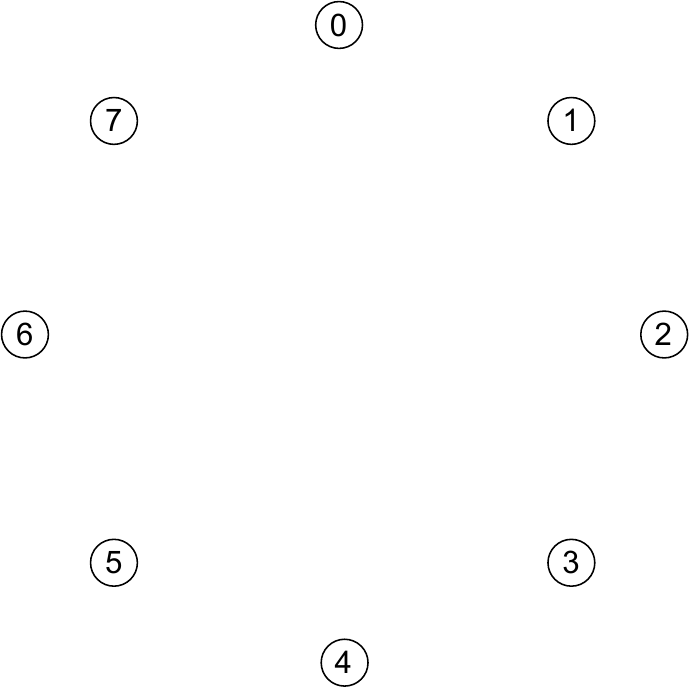}
}
\hfill
\subfloat[\textsc{LeadingOnes}]{
	\includegraphics[width=0.22\textwidth]{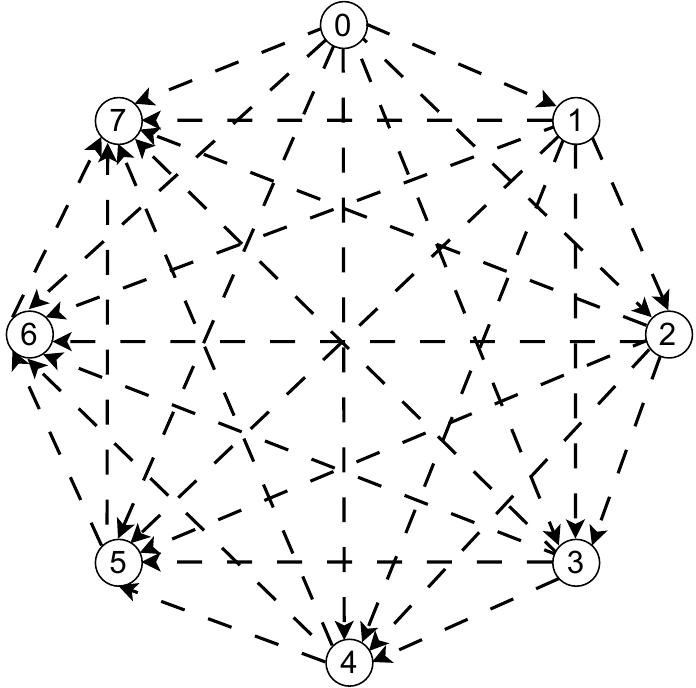}
}
\hfill
\subfloat[\textsc{CTrap}]{
	\includegraphics[width=0.22\textwidth]{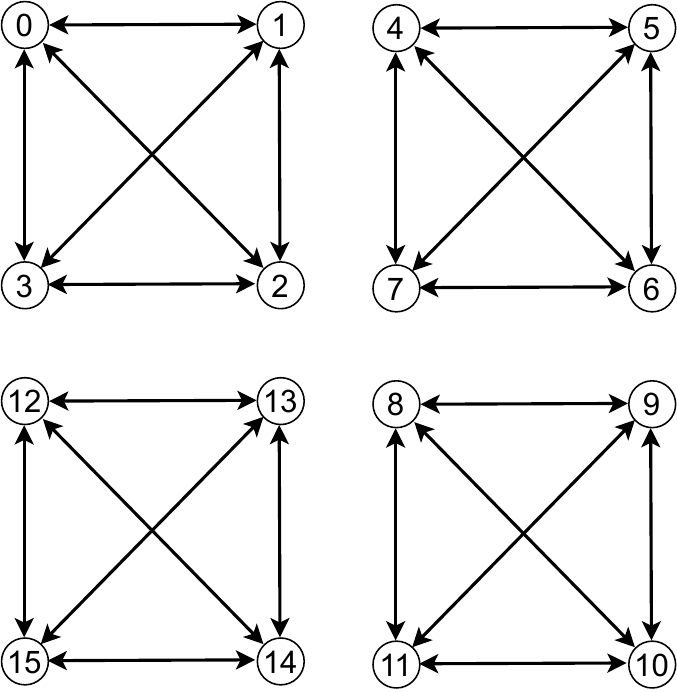}
}
\hfill
\subfloat[\textsc{CNiah}]{
	\includegraphics[width=0.22\textwidth]{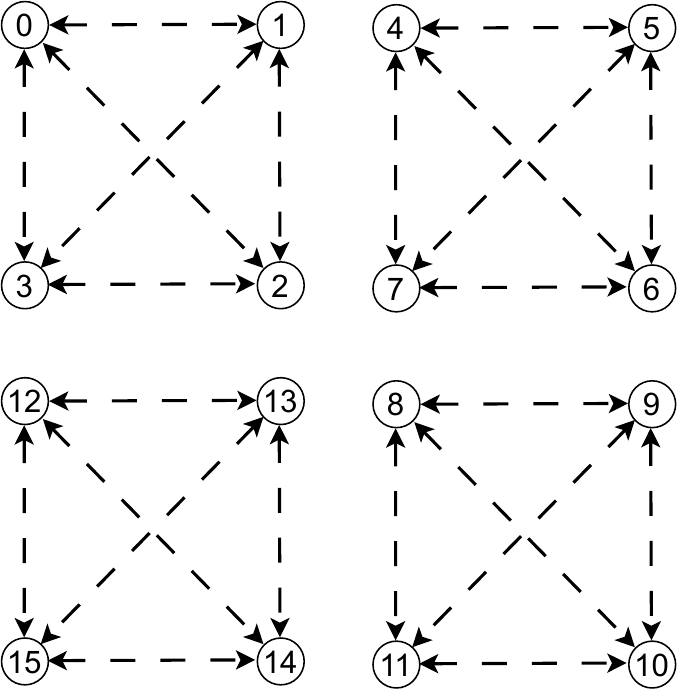}
}\\

\hfill
\subfloat[\textsc{CycTrap}\label{fig:cyc-eg}]{
	\includegraphics[width=0.22\textwidth]{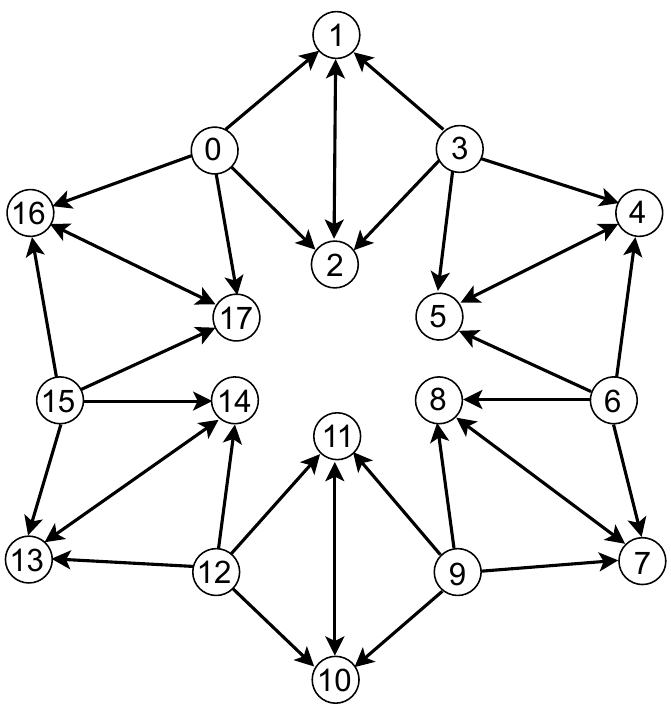}
}
\hfill
\subfloat[\textsc{LeadingTraps}]{
\includegraphics[width=0.22\textwidth]{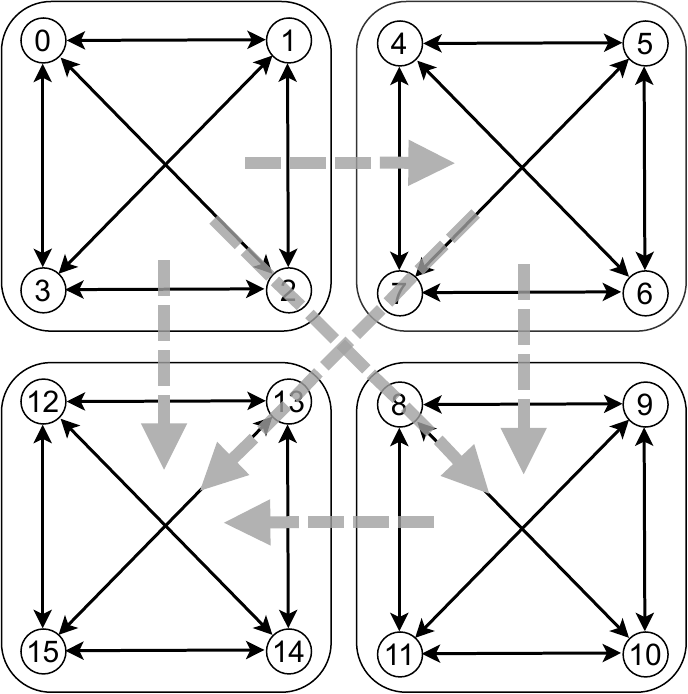}
}
\hfill
\subfloat[\textsc{MaxSat}: \((x_0\vee x_2 \vee x_3) \wedge (x_0 \vee \neg x_3) \wedge (\neg x_0 \vee x_3 \vee x_6) \wedge (\neg x_0 \vee x_5) \wedge (x_1 \vee x_2) \wedge (\neg x_1 \vee \neg x_2 \vee x_7) \wedge (\neg x_2) \wedge (\neg x_2 \vee x_3) \wedge (\neg x_3) \wedge x_4 \wedge (x_5 \vee \neg x_6 \vee x_7) \wedge \neg x_7 \)]{
	\hspace{1.6cm}
	\includegraphics[width=0.22\textwidth]{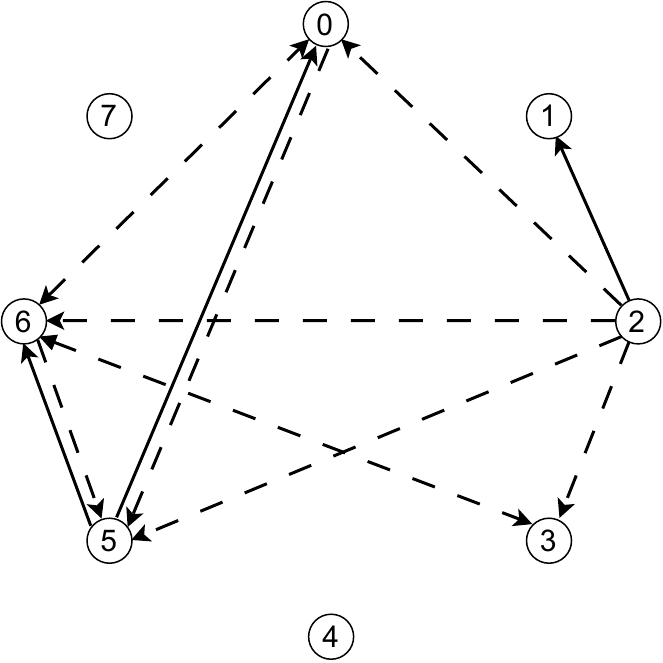}
	\hspace{1.6cm}
}
\caption{Epistasis graphs of the exemplified problems. }
\label{fig:EG}
\end{figure}

An EG is essentially a collection of binary relations defined by epistasis. Adopting binary relations to describe higher-order relations is not new in the field of MBGA. To name a few, the Bayesian network in BOA~\cite{Pelikan:boa:1999}/hBOA~\cite{Pelikan:hboa:2001}/ENBA~\cite{Etxeberria:enba:1999}, the dependency structure matrix in DSMGA-I~\cite{Yu:dsmga:2009}/II~\cite{Hsu:dsmga2:2015}, the linkage tree detection in LT-GOMEA~\cite{Pelikan:pairwise:2011}, and the variable-interaction graph in gray-box optimization~\cite{Whitley:vig:2016} are all binary relations. Their mechanisms/definitions are somewhat different. Here we adopt binary relations with strong reasons, and the definitions enable the derivations in the rest of this paper.

\section{Problem Decomposition and\\Epistatic Graph}
\label{sec:eg}

Modern MBGAs utilize models to achieve problem decomposition. Based on the schema theorem~\cite{Holland:1975}, short, low-order, and highly fit schemata are favored during the GA process and are considered as \textit{building blocks} to be recombined for obtaining potentially better  solutions~\cite{Goldberg:1989:ga-book}. Equipped with the formalized mathematical constructs, established according to the nature and properties of the underlying problem by using EGs, proposed in Section~\ref{sec:EpistasisProperties}, we are now enabled to formalize the concept of building blocks, and on such a ground, whether or how a problem may be decomposed can be discussed and quantified. Thus, we may be able to fill a gap in the theoretical development of GAs in this regard and provide a foundation for further research investigations and GA design guidelines. 

This section proposes the problem decomposition theorem which strongly relates problem decomposition to EGs. Specifically, the in-closures of EGs provide a foundation of atomic components for problem decomposition. Then this section shows that there exists a way of decomposition such that the global optimum can be found in polynomial time if the sizes of the strongly connected components (SCCs) in the EG of the given problem are bounded by the logarithm of the problem size.

\subsection{Stationary Optima}
The nature of GAs rewards promising fragments during recombination by providing the recombined chromosomes with different survival rates according to the fitness function. If we take this idea to the extreme, we can consider those fragments that do not reduce fitness no matter which chromosomes they are pasted to. We call such a fragment a stationary optimum formally defined as follows.
\begin{definition}
	A \textbf{stationary optimum} (SO) is a nonempty assignment $A$ with the following property: No matter how one assigns the alleles unspecified by $A$, the gene assignments on the loci specified by $A$ remain \textit{superior} to all other assignments on the same coverage.\label{def:SO} \\	
	Mathematically, for all assignments $A' \neq A$ where $\mathcal{C}(A')=\mathcal{C}(A)$ and for all remaining assignments $R$ where $\mathcal{C}(A) \cup \mathcal{C}(R) = V$, we have $f(A\cup R) > f(A'\cup R)$. 
\end{definition}

An obvious stationary optimum is the global optimum. In general, the stationary optimum may not be unique; however, it is unique under the assumption that the global optimum is unique (Assumption~\ref{asm:unique}). Following Definition~\ref{def:SO}, we can know that for an SO, for any locus, the allele is either unassigned or assigned the allele of the global optimum:
\begin{proposition}
	For any SO $A$, $\forall v, A[v] = g[v]$ or *.
	\label{prop:so-g}
	\begin{proof}
	Suppose that $A[v] = \overline{g[v]}$ for some $v \in V$.\\
	Let $W = V - \mathcal{C}(A)$ be the set of remaining loci. Since $\{(V, g)\}$ is the unique global optimum, $f( \{(\mathcal{C}(A), g)\} \cup \{(W, g)\} ) > f(A \cup \{ (W, g) \})$, which makes $A$ not stationarily optimal.
	\end{proof}
\end{proposition}	

Conceptually, SOs provide as fragments of problem decomposition in optimization. The definition of SOs makes they can be optimized separately. Once the loci in an SO are correctly assigned, the epistases in the remaining subproblem remain unchanged. We call the next theorem the separation property of SOs.

\begin{theorem}[Separation property of stationary optima]
Suppose that assignment $A$ is an SO. $\forall R \cap A = \phi$,
	$\psi_R = \psi_{R\cup A}$.\label{thm:sep}
\begin{proof}
Since $A$ is stationarily optimal, $\forall R \cap A = \phi$ and $\forall s\in A$, we have $\psi_R[s] = \{g[s]\}$. By Proposition~\ref{prop:opt-remain1}, we have $\psi_R = \psi_{R\cup\{(s,g)\}}$. We can then continue to union the gene assignments in $A$, and eventually we have $\psi_R = \psi_{R\cup \{(\mathcal{C}(A),g)\}} = \psi_{R\cup A}$.
\end{proof}
\end{theorem}

The separation property states that any constrained optimum remains unchanged with additional constraint $\{(\mathcal{C}(A),g)\}$ as long as $A$ is stationarily optimal. Since epistasis is defined based on constrained optima, epistatic relations remain unchanged with a stationarily optimal assignment, and so does the EG.

\begin{corollary}
Suppose that assignment $A$ is an SO of the target problem. The EG of the subproblem with the loci in $\mathcal{C}(A)$ being correctly assigned is the subgraph defined on $V-\mathcal{C}(A)$ of the EG of the target problem. \label{coro:sep}
\end{corollary}

Next, we derive that in-closures of any loci are SOs if the underlying problem does not contain weak epistasis.

\begin{lemma}
	If the problem does not contain weak epistasis, for all $v \in V$, the assignment $A = \{ \mathcal{(IN}^*(v), g)\}$ is stationarily optimal.\label{lemma:in-so}
	\begin{proof}
		Suppose that $A$ is not stationarily optimal.\\
		There exists an assignment $R$ to loci unspecified by $A$, \textit{i.e.}, $\mathcal{C}(R) \cap \mathcal{IN}^*(v) = \phi$, such that the constrained optima with constraint $R$ differ from $A$ at some loci specified in $A$. Let $w$ be one such locus: $w \in \mathcal{C}(A)$  and $\Psi_R[w] \neq \{A[w]\}$. Since $w$ is assigned to $g[w]$ in $A$, we have $\overline{g[w]} \in \Psi_R[w]$.

		Given that all epistases in the problem are non-weak, by Proposition~\ref{prop:someone}, $\exists r \in \mathcal{C}(R)$, $\{r\}\Rightarrow w$. Since $w \in \mathcal{IN}^*(v)$ and $r \not \in \mathcal{IN}^*(v)$, that leads to a contradiction.
	\end{proof}
	
\end{lemma}

Next, we are interested in the \emph{minimum stationary optimum} (MSO) defined as follows.
\begin{definition}
	The \textbf{minimum stationary optimum} of locus $v$, denoted by $\mathcal{M}_{SO}(v)$, is a stationary optimum with $v$ assigned, having the minimum size. 
\end{definition}

For example, in \textsc{OneMax}, $\mathcal{M}_{SO}(v) = \{(v,1)\}$ for all $v\in V$. In the 8-bit \textsc{CTrap}, $\mathcal{M}_{SO}(0) = \mathcal{M}_{SO}(1) = \mathcal{M}_{SO}(2) = \mathcal{M}_{SO}(3) = \{(0,1), (1,1), (2,1), (3,1)\}$, and $\mathcal{M}_{SO}(4) = \mathcal{M}_{SO}(5) = \mathcal{M}_{SO}(6) = \mathcal{M}_{SO}(7) =\{(4,1), (5,1), (6,1), (7,1)\}$. One may think that multiple assignments may fit the definition of $\mathcal{M}_{SO}(v)$. However, the following theorem states that under our assumptions, $\mathcal{M}_{SO}(v)$ is unique for each $v \in V$ since its coverage is the in-closure of $v$. This is our first major result.

\begin{theorem} [Problem Decomposition Theorem]
	If the problem does not contain weak epistasis, $\mathcal{C}(\mathcal{M}_{SO}(v)) = \mathcal{IN}^{*}(v)$ for all $v\in V$.
	\label{thm:m=in}
	\begin{proof}
		
	First, we prove by contradiction the fact that for any stationary optimum $A$, $\forall v\in \mathcal{C}(A), u \not \in \mathcal{C}(A)$, we have $\{u\} \not\Rightarrow v$.
	Suppose that there exist $v\in \mathcal{C}(A)$ and $u \not \in \mathcal{C}(A)$ such that $\{u\} \Rightarrow v$. Because $\Psi_{ \{(u, g) \} }[v] = \{ g[v]\}$ (Proposition~\ref{prop:opt}), by the definition of epistasis, we have $\overline{g[v]} \in \Psi_{ \{(u, \overline{g}) \} }[v]$. We also know that $A[v] = {g[v]}$ by Proposition~\ref{prop:so-g}.
	Since the definition of SO requires the assignment $(v, g[v])$ to be strictly superior, we have a contradiction.
	
	Next, we claim that $\mathcal{IN}^{*}(v) - \mathcal{C}(\mathcal{M}_{SO}(v)) = \phi$. Otherwise, there exist $u \not\in \mathcal{C}(\mathcal{M}_{SO}(v))$ and $w \in \mathcal{C}(\mathcal{M}_{SO}(v))$ such that $\{u\} \Rightarrow w$, which contradicts with the aforementioned fact since $\mathcal{M}_{SO}(v)$ is a stationary optimum. As a result, $\mathcal{IN}^{*}(v) \subseteq \mathcal{C}(\mathcal{M}_{SO}(v))$. On the other hand, we know that $\{ \mathcal{IN}^*(v), g \}$ is an SO by Lemma~\ref{lemma:in-so}, which means $ \mathcal{C}(\mathcal{M}_{SO}(v)) \subseteq \mathcal{IN}^{*}(v)$. Putting them together makes $\mathcal{C}(\mathcal{M}_{SO}(v)) = \mathcal{IN}^{*}(v)$.	
	\end{proof}
\end{theorem}

Under the assumption that the global optimum is unique, we also know the assignment on $\mathcal{M}_{SO}(v)$.
\begin{corollary}
	$\mathcal{M}_{SO}(v) = \{ (\mathcal{IN}^*(v), g) \}$.
\end{corollary}

We now apply the corollary to our exemplified test problems. In \textsc{OneMax}, each in-closure is simply one single vertex, and hence $\mathcal{M}_{SO}(v)=\{(v,1)\}$. In \textsc{LeadingOnes}, the in-closure of $v$ consists of all its preceding vertices including itself, and hence $\mathcal{M}_{SO}(v)=\{(\{0,1,\ldots, v\},1)\}$.
In \textsc{Ctrap} and \textsc{CNiah}, the in-closures are the maximal cliques, and hence $\mathcal{M}_{SO}(v)=\{ (\{i, i+1, i+2, i+3   \},1)\}$, where $i = 4\lfloor v/4\rfloor$. In \textsc{LeadingTraps}, the in-closures are similar to those in \textsc{LeadingOnes}, but in modulo 4; 
$\mathcal{M}_{SO}(v)=\{ (\{0, 1, \ldots, 4\lfloor v/4\rfloor+3   \},1)\}$. As for \textsc{CycTrap}, the situation is a bit complicated. Taking a look at Figure~\ref{fig:EG}(e), one may think $\mathcal{M}_{SO}(2) = \{(\{0, 1, 2, 3\}, 1)\}$, where the coverage is the in-closure of locus 2, according to the corollary. However, it requires at least three blocks for a set to be stationarily optimal, \textit{e.g.}, $\mathcal{M}_{SO}(2)=\{ (\{0, 1, \ldots, 9\},1)\}$, whose coverage consists of $4+3+3=10$ loci. This is because \textsc{CycTrap} consists of weak epistases of order 2. For example, $\{2,4\}\Rightarrow 3$, while $\{2\}\not\Rightarrow 3$ and $\{4\}\not\Rightarrow 3$. Nevertheless, once we have an MSO (of 10 loci); with additional 3 loci from the adjacent block, it forms another SO (of 13 loci); again, with another additional 3 loci from the adjacent block, it forms yet another SO (of 16 loci). This observation gives a strong hint of problem decomposition for \textsc{CycTrap}, which is later addressed in Section~\ref{sec:eg-summary}.

We have seen that MSOs are disjoint as in \textsc{OneMax}, \textsc{CTrap}, and \textsc{CNiah}, or that one is a subset of another as in \textsc{LeadingOnes} and \textsc{LeadingTraps}. MSOs can overlap as well. Consider the fitness shown in Figure~\ref{fig:abac}. In this problem, $a\rightarrow b$ and $a\rightarrow c$. $\mathcal{M}_{SO}(a) = \{(a,1)\}$, $\mathcal{M}_{SO}(b) = \{(a,1), (b,1)\}$, and $\mathcal{M}_{SO}(c) = \{(a,1), (c,1)\}$. In this case, $\mathcal{M}_{SO}(b)$ and $\mathcal{M}_{SO}(c)$ partially overlap.

\begin{figure}
	\centering
	\hspace{2.5cm}	
	\includegraphics[width=0.3\columnwidth]{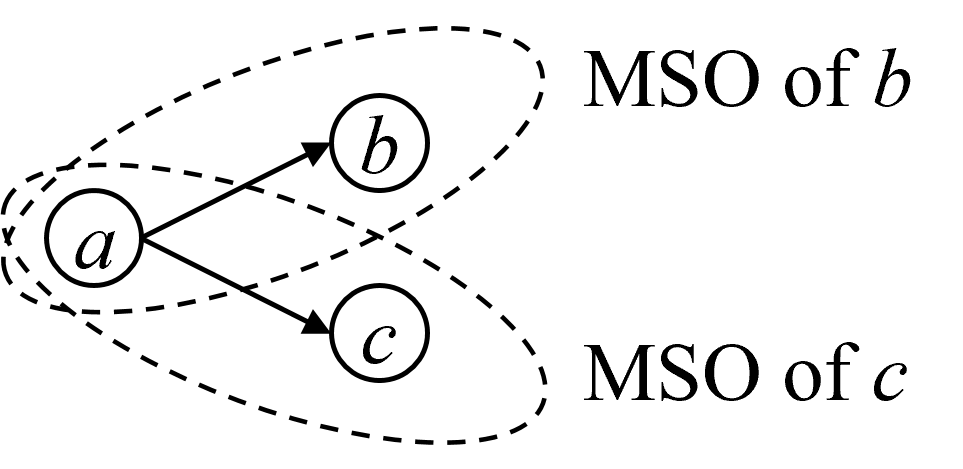}
	\hfill	
	\small{
	\begin{tabular}{cccc}
		\toprule
		\textbf{a} & \textbf{b} & \textbf{c} & \textbf{Fitness} \\ \hline
		1 & 1 & 1 & 10 \\
		\hline
		1 & 1 & 0 & 9  \\
		1 & 0 & 1 & 9  \\
		\hline
		1 & 0 & 0 & 8  \\
		\hline
		0 & 0 & 0 & 7  \\
		\hline
		0 & 1 & 0 & 6  \\
		0 & 0 & 1 & 6  \\
		\hline
		\multicolumn{3}{c}{others} & $<6$ \\ \bottomrule
	\end{tabular}	
    }
	\hspace{2cm}
	\vspace{6px}
	\caption{A problem with overlapping MSOs. In this problem, $a\rightarrow b$ and $a\rightarrow c$. $\mathcal{IN}^*(a)=\{a\}$, $\mathcal{IN}^*(b)=\{a,b\}$, and $\mathcal{IN}^*(c)=\{a,c\}$.}\label{fig:abac}
\end{figure}
The previous section elaborated that under certain assumptions, important characteristics can be captured by binary relations between genes, and subsequently EG is defined. Here, the problem decomposition theorem (Theorem~\ref{thm:m=in}) plays an important role in gluing two concepts together: problem decomposition and the epistatic graph. To the best of our knowledge, this is the first formal proof showing that problem decomposition is related to the concept of linkage in general.

\subsection{Problems That Can Be Solved via Proper Decomposition in Polynomial Time with Known EGs}
\label{sec:eg-summary}
Note that even if the problem is decomposable, MSO can still be large. For example, \textsc{LeadingOnes} can be solved bit-by-bit in the correct order; while $\mathcal{C}(\mathcal{M}_{SO}(0)) = \{0\}$ and $\mathcal{C}(\mathcal{M}_{SO}(99)) = \{0,1,\ldots, 99\}$. This observation strongly suggests that the algorithm should start from the smallest MSO. With the EG provided, such a proper order can be decided. This subsection discusses those problems that can be solved in polynomial time if their EGs are known.

In an \textit{additively separable} problem, subproblems are disjoint~\cite{Eiben:critique:1999}. One of the characteristics of additively separable problems is that the global optimum can be found via partial enumeration in polynomial time with the problem structure known in advance. For example, if the problem consists of $m$ disjoint subproblems of $k$ bits, $2^km$ function evaluations are sufficient to obtain the global optimum with $2^k$ evaluations on each subproblem.

The above idea can be carried on problems that are not additively separable to some extent. Let tuple $D = (d_0, d_1, \ldots, d_{|D|-1} )$ be a partition of set $V$. Consider the partial enumeration algorithm (PE) shown in Algorithm~\ref{alg:pe}, which takes $\sum_{i=0}^{|D|-1} 2^{|d_i|}$ function evaluations. With the input $D=(V)$, the algorithm reduces to a full enumeration algorithm that obtains the global optimum by $2^\ell$ function evaluations, which is exponential to the problem size. Next, we show that under certain conditions, PE returns the global optimum in polynomial time with an appropriate partition $D$.

\begin{algorithm}
	\caption{Partial Enumeration (PE)}
	\label{alg:pe}
	
	\KwIn{$D = (d_0, d_1, \ldots, d_{|D|-1})$: a partition of $V$}
	\KwOut {A chromosome that aims to be the global optimum}
	
	Randomly initialize a chromosome $\vec y$.\\
	\For{$i=0, 1, \ldots, |D|-1$} {
		\For{each $A$ in all $2^{|d_i|}$ possible assignments}{
			\If{$f(\vec{y}^A) > f(\vec y)$} {
				$\vec y \leftarrow \vec y^A$
			}
		}
	}
	\textbf{return} $\vec y$\\
\end{algorithm}		

\subsubsection{Acyclic EG}
\label{sec:DAG}
For problems whose EGs are DAGs, let $D
= (\{v_0\}, \{v_1\}, \ldots, \{v_{|V|-1}\})$ be a partition of $V$, where $v_i\in V$ follows a topological order: $\forall 0 < i < |V|$, $\{v_{i}\} \not\Rightarrow v_{i-1}$. PE takes $2\ell$ function evaluations on the input $D$. We claim that PE returns the global optimum on such an input $D$ as well.

\begin{theorem}
The allele of $v_i$ is assigned to $g[v_i]$ for all $i$ in PE on input   $D = (\{v_0\}, \{v_1\}, \ldots, \{v_{|V|-1}\})$ where $\forall 0< i < |V|$, $\{v_{i}\} \not\Rightarrow v_{i-1}$.\label{thm:acyc}
\begin{proof}
Prove by induction.
\begin{itemize}
\item Base: $v_0$ is correctly set to $g[v_0]$ by PE since $\mathcal{C}(\mathcal{M}_{SO}(v_0)) = \{ v_0 \}$.
\item Induction hypothesis: $v_j$ is set to $g[v_j]$ for all $j < i$.
\item Now consider $v_i$. Let $V_i = \bigcup_{j\leq i} v_j$. Since $\mathcal{C}(\mathcal{M}_{SO}(v_i)) = \mathcal{IN}^*(v_i)$, we have $\mathcal{M}_{SO}(v_i) \subseteq V_i$. When PE decides the allele on $v_i$, all other alleles in $\mathcal{M}_{SO}(v_i)$ have been correctly decided (induction hypothesis), so $v_i$ is correctly assigned to $g[v_i]$.
\end{itemize}
By the principle of mathematical induction, the statement holds for all $i \geq 0$.
\end{proof}
\end{theorem}

The problems \textsc{OneMax}  and \textsc{LeadingOnes} fall into this category (Figure~\ref{fig:EG}). In \textsc{OneMax}, every gene is independent of the others. The order does not matter, and hence PE returns the global optimum. In \textsc{LeadingOnes}, according to the topological order of the EG, PE decides the alleles from the first gene to the last gene, and that results in the global optimum. 

\subsubsection{EG with boundedly sized SCCs}
\label{sec:SCC}
Next, we show that problems can be solved in polynomial time if the SCCs in their EGs are with bounded sizes. This is because, after graph condensation, the resulted component graph is a DAG, we can then use a similar approach as in Section~\ref{sec:DAG}. The following theorem formally states such a property. Figure~\ref{fig:condensation} illustrates the component graphs of \textsc{CTrap}, \textsc{CNiah}, \textsc{CycTrap}, and \textsc{LeadingTraps}. Since there does not exist any SCC of size greater than 1 in the EGs of \textsc{OneMax} and \textsc{LeadingOnes}, their component graphs are identical to their EGs.

\begin{figure*}
\centering
\subfloat[Component graph of \textsc{CTrap}/\textsc{CNiah}.]{
\includegraphics[width=0.28\textwidth]{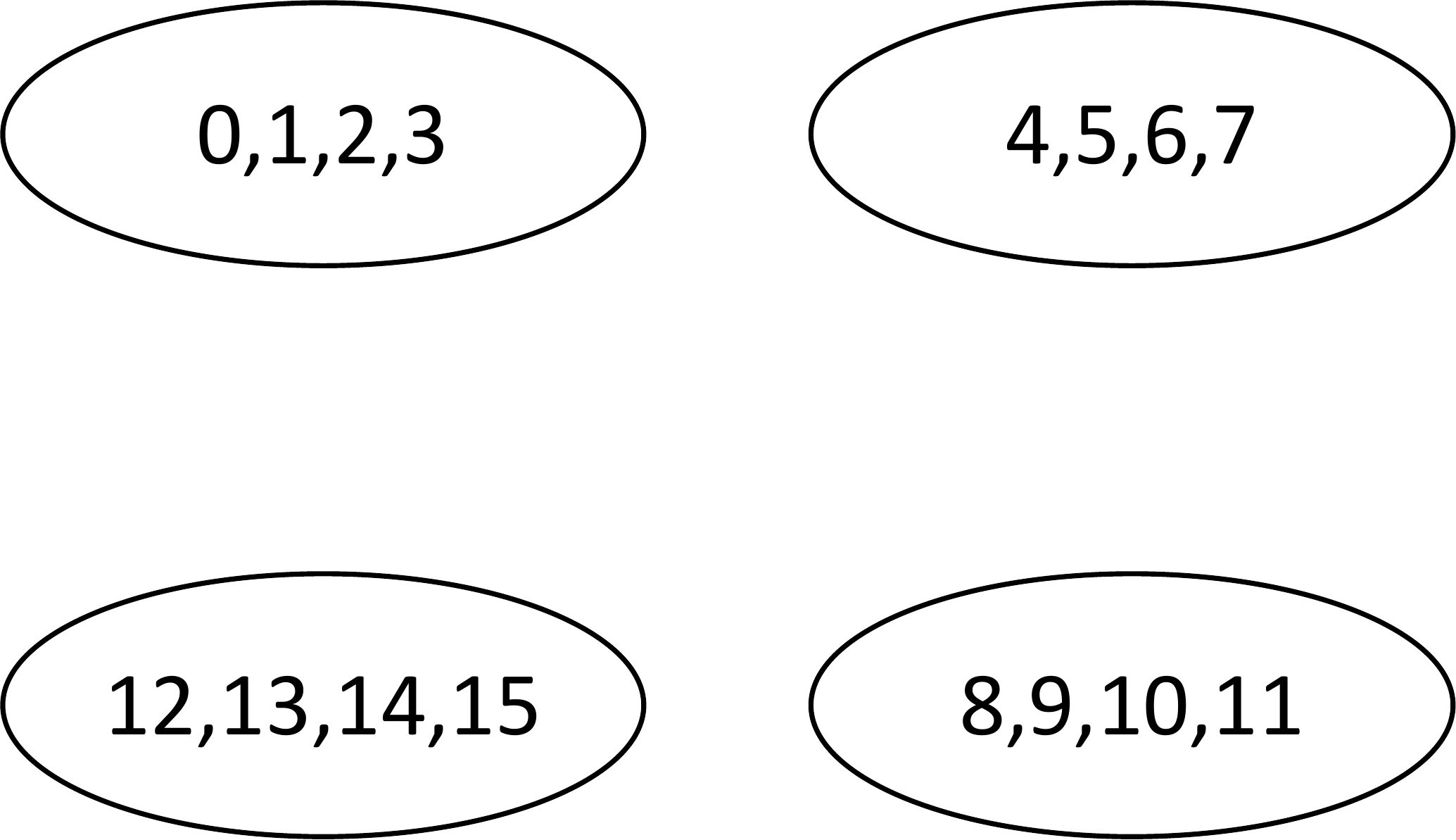}
}
\hfill
\subfloat[Component graph of \textsc{CycTrap}.]{
	\includegraphics[width=0.38\textwidth]{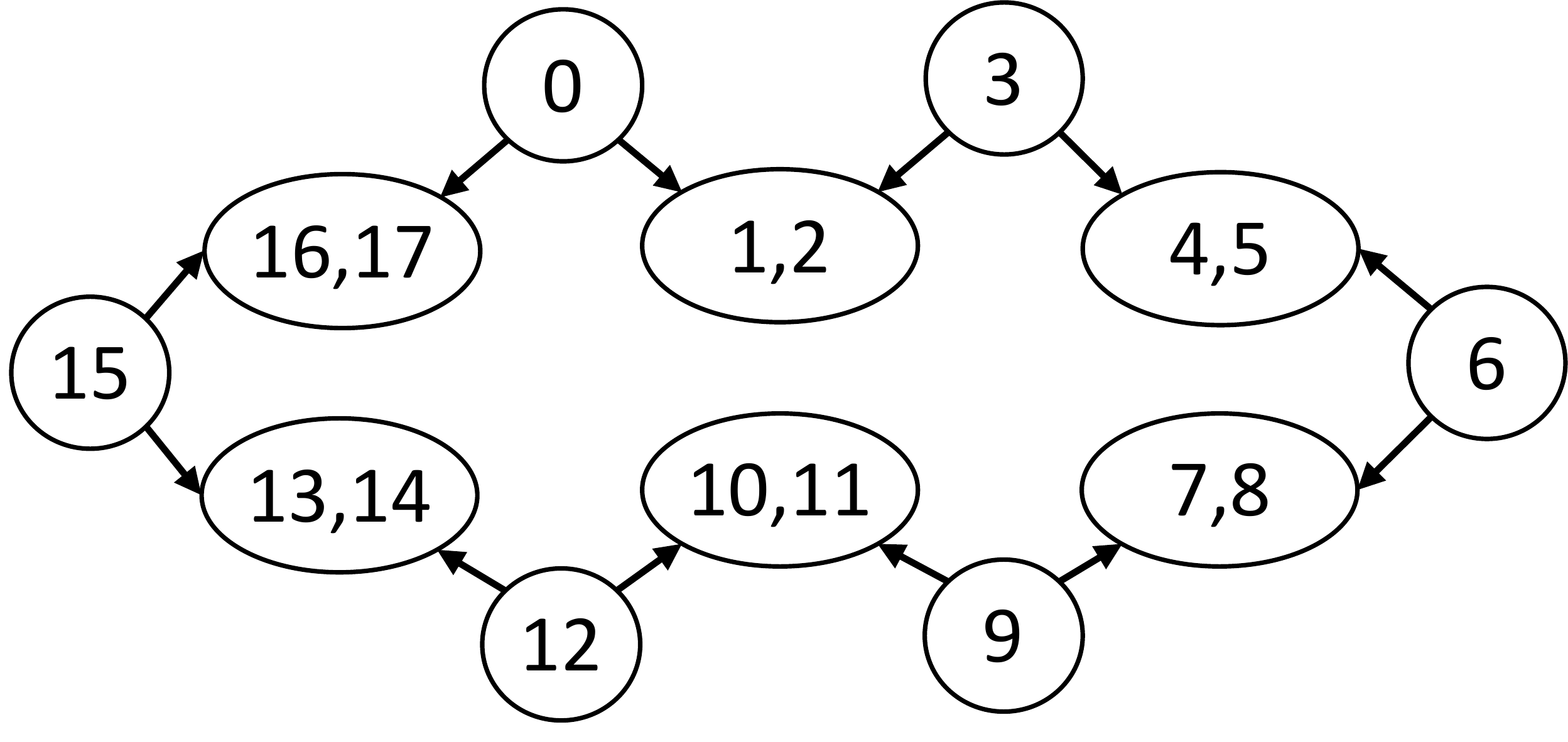}
}
\hfill
\subfloat[Component graph of \textsc{LeadingTraps}.]{
	\includegraphics[width=0.28\textwidth]{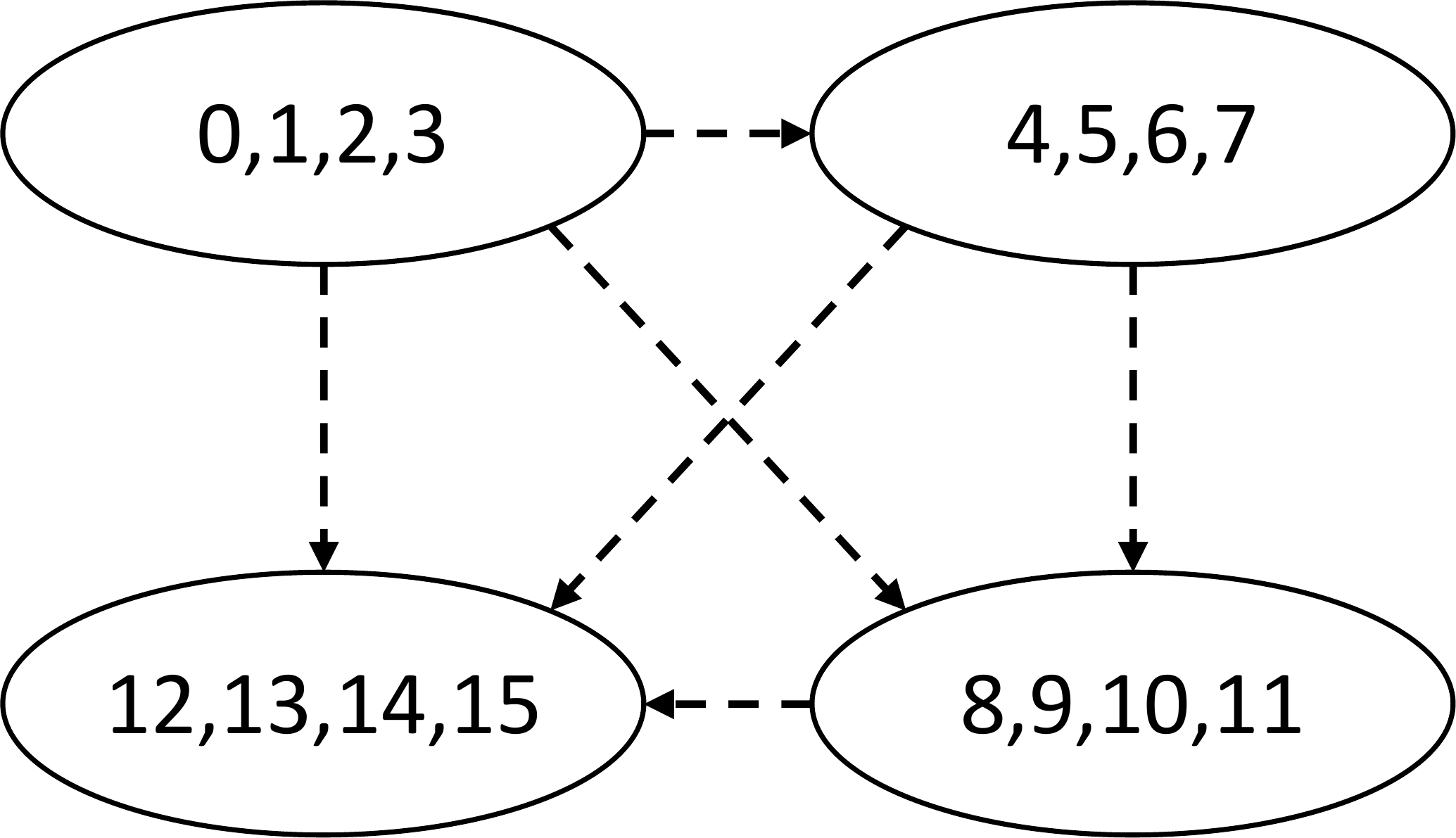}
}
\caption{Graph condensation resulted in the component graph, which is a DAG.}
\label{fig:condensation}
\end{figure*}

\begin{theorem}
	If the SCCs in the EG are with the sizes of $O(\log \ell)$, there exists a partition $D$ of set $V$ for PE to return the global optimum within $\ell^{O(1)}$ function evaluations.	
	\label{thm:np}
	\begin{proof}
		Let $G_c = (V_c, E_c)$ be the component graph of the EG with all SCCs contracted to vertices. Since $G_c$ is a DAG, let $Q
		= (q_0, q_1, \ldots, q_{|V_c|-1})$ be a topological order of $V_c$ such that $\forall 0< i < |V_c|$, $(q_{i}, q_{i-1})\not\in E_c$. Finally, let $D = (d_0, d_1, \ldots, d_{|V_c|-1})$ be the partition of $V$ corresponding to $Q$, \textit{i.e.}, the contraction of $d_i$ is $q_i$ for all $i$. Next, we prove the following properties.
		\begin{enumerate}
			\item  The number of function evaluations consumed by PE with $D$ as input is $\ell^{O(1)}$. \\
			Each $d_i$ corresponds to one single vertex in the component graph $G_c$, where each vertex represents either one vertex or one SCC in EG. Therefore, $|d_i| = O(\log \ell)$. 	
			The number of function evaluations in the inner loop of Algorithm~\ref{alg:pe} is $O(2^{|d_i|}) = \ell^{O(1)}$. Given that $|D| = O(\ell)$, the total number of function evaluations is then $\ell ^ {O(1)}$.
			\item The algorithm returns the global optimum.\\	
			Since $Q$ follows a topological order and $D$ is its correspondence, the alleles in the component graph are correctly assigned by Theorem~\ref{thm:acyc}. As a result, the algorithm returns the global optimum. 
		\end{enumerate}
		
		From what has been mentioned above, the statement holds.	 
	\end{proof}
\end{theorem}

In summary, if the SCCs of EG have sizes that are bounded by the logarithm of the problem size, the global optimum can be obtained via partial enumeration in polynomial time with an appropriate order. Since the sizes of SCCs in an acyclic graph are either 0 or 1 depending on the definition of reachability, the above conclusion also applies to acyclic graphs. We believe the concept here reflects the term ``nearly decomposable problems with bounded difficulty''~\cite{Pelikan:decompose:2006, Goldberg:doi:2002}. In addition, the partitions of the component graph (SCCs) correspond to the idea of building blocks~\cite{Goldberg:1992:ga-decomposition,Goldberg:doi:2002}.

Note that \textsc{CycTrap} contains weak epistases, so the way of constructing a sequence of partitions for PE via component graphs does not directly apply to \textsc{CycTrap}. Nevertheless, as pointed out before, the partition $(\{0,1,\ldots, 9\}, \{10,11,12\}, \{13, 14, 15\}, \ldots)$ still enables PE to find the global optimum for  \textsc{CycTrap}. This observation indicates that some non-additively separable problems can still be solved using a decomposition method that appears to work only on additively separable problems. Such a message is also carried on \textsc{LeadingTraps}, in which case the decomposition can be fully explained by the tools developed in this paper. However, finding an appropriate partition for problems like \textsc{CycTrap} requires further investigation into weak epistases, which is beyond the scope of this paper.

One may find that the setting in this section resembles that of NP. If we partition $V$ in a non-deterministic way, there exists a way to decompose the problem such that the global optimum can be found in polynomial time. Alternatively, we can view the way of constructing $D$ in the proof of Theorem~\ref{thm:np} as an \textit{oracle}.  As shown in this section, knowing EG in advance provides a sufficient condition for such an oracle to be realizable; further investigation is needed for the necessary conditions.
Nevertheless, the next section shows that without the oracle, or the EG, the global optima of such problems are PAC learnable with additional requirements. 

\section{Global Optima of Nearly Decomposable Problems with Bounded Difficulty Are PAC Learnable}
\label{sec:learnability}

\noindent The previous section proved the existence of the partial enumeration sequence to solve nearly decomposable problems in polynomial time, or equivalently, assuming the problem structure (EG) is known in advance. In this section, we show that even without knowing the problem structure, the global optima of nearly decomposable problems with bounded difficulty
can still be found in polynomial time and are PAC learnable~\cite{Valiant:pac:1984}.

\subsection{Extreme Balanced Accuracy and PAC}
\label{sec:pac}
To relate optimization with PAC, we need to state clearly what the hypothesis space is. 
\begin{definition}
A \textbf{hypothesis} is a function that takes a chromosome as input, outputs \textsc{True} if the chromosome is globally optimal, and outputs \textsc{False} otherwise: $h: X \rightarrow B$, where chromosome $\vec x \in X$ and $B=\{\textsc{True},\textsc{False}\}$. The \textbf{hypothesis space} $H$ is a set of hypotheses.
\end{definition}
Since in optimization, the two classes (global optima or not) are usually imbalanced, we consider the \textit{balanced accuracy} (BACC) in the PAC definition:
\begin{linenomath*}
\begin{equation*}
BACC = \frac{1}{2} (sensitivity + specificity).
\end{equation*}
\end{linenomath*}

In addition, since most optimization algorithms terminate when \textit{any} one of the global optima is reached, we define the \textit{extreme sensitivity} based on the ordinary sensitivity:
\begin{linenomath*}
\begin{equation*}
sensitivity^* = [sensitivity >0],
\end{equation*}
\end{linenomath*}
where $[Q]$ is the Iverson bracket, which is 1 when $Q$ is \textsc{True} and 0 when $Q$ is \textsc{False}. The extreme balanced accuracy (EBACC) is then defined as 
\begin{linenomath*}
\begin{equation*}
EBACC = \frac{1}{2} (sensitivity^* + specificity).
\end{equation*}
\end{linenomath*}

Under the assumption where the global optimum is unique (Assumption~\ref{asm:unique}), EBACC of any hypothesis that incorrectly identifies the unique optimum is less than $\frac{1}{2}$ (because $sensitivity^*=0$ and $specificity<1$). EBACC of the hypothesis that correctly identifies the unique optimum is 1.

Originally, a concept is PAC learnable if there exists an algorithm that outputs a hypothesis with an average error less than or equal to $\epsilon$ with probability $1-\delta$ in time polynomial to $\frac{1}{\epsilon}$, $\frac{1}{\delta}$, and the size of samples~\cite{Valiant:pac:1984}.  With the definition of EBACC, the global optimum of an optimization problem is PAC learnable if there exists an algorithm that achieves EBACC of $1-\epsilon$ with  probability $1-\delta$ consuming (fitness) function evaluations polynomial to $\frac{1}{\epsilon}$, $\frac{1}{\delta}$, and the problem size $\ell$.

\subsection{The Global Optimum Is PAC Learnable}
Consider a routine \textsc{TestSO} (Algorithm~\ref{alg:testso}) that tests via enumeration whether a given set $S$ is the coverage of an SO based on a population $P$ with $n$ randomly generated chromosomes. It returns \textsc{False} if $S$ is not and returns \textsc{True} with the optimal assignment if $S$ is. Intuitively, the larger $n$ is, the more accurate the test can be. \textsc{TestSO} always returns \textsc{True} on SOs by definition.

\begin{algorithm}
	\caption{\textsc{TestSO}($S$, $P$)}
	\label{alg:testso}
	
	\KwIn{$S \subseteq V$; $P$: a population containing $n$ chromosomes}
	\KwOut{\textsc{True}/\textsc{False} and an assignment on $S$}
	
	\For {each chromosome in $P$} {
		Enumerate all $2^{|S|}$ assignments on $S$ in the chromosome; record the assignment that is superior to all others as $A$; \textbf{return} (\textsc{False}, $\phi$) if no such assignment exists.\\
		\If {$A$ changes between iterations}{\textbf{return} (\textsc{False}, $\phi$)}
	}
	\textbf{return} (\textsc{True}, $A$)\\
\end{algorithm}

\begin{proposition}
If the input $S$ is the coverage of an SO (including MSO), \textsc{TestSO} always returns \textsc{True} with the correct assignment $A$ for any $n \geq 1$. Consequently, if $S$ is the coverage of an SO minus a set $S'$ and the loci in $S'$ are correctly assigned in all $n$ chromosomes, \textsc{TestSO} always returns \textsc{True}.
\label{prop:so}
\end{proposition}

Algorithm~\ref{alg:ie} sketches an iterative partial enumeration algorithm (IPE). The idea is to try all combinations of loci in ascending order by size; assign and freeze the alleles of the combination as the stationary superior pattern if such a pattern is  found; the procedure does not terminate until all combinations of unassigned loci are tested. If \textsc{TestSO} keeps failing to find the superior pattern, the value of $k$ keeps increasing until $k=\ell$, which causes \textsc{TestSO} to perform a full enumeration and return with the global optimum assignment.
Also, early incorrect assignments may cause \textsc{TestSO} to fail to find any SO on the remaining loci, which causes IPE to return \textsc{Failure}.
We will show that this is unlikely to happen on nearly decomposable problems with bounded difficulty when $n$ is sufficiently large.

\begin{algorithm}
	\caption{Iterative Partial Enumeration (IPE)}
	\label{alg:ie}
	
	\KwIn{$n$: the number of samples}
	\KwOut{a full assignment or \textsc{Failure}}
	
	Generate a population $P$ with $n$ chromosomes by uniformly randomly assigning alleles.\\	
	$U \leftarrow V$\\
	$k \leftarrow 1$\\
	\While{$k \leq |U|$} {
		\For{each $S \subseteq U$ where $|S|=k$} {
			($foundSO$, $A$) = \textsc{TestSO}($S, P$)\\
			\If {foundSO is \textsc{True}} {
				Apply assignment $A$ to all $n$ chromosomes in $P$.\\
				$k \leftarrow 1$\\
				$U \leftarrow U-S$\label{alg-line:U}\\
				\If {$U$ is empty} {
					\textbf{return} the assignment of a chromosome in $P$.
				}
			}			
		}
		\If {none of $S\subseteq U$ where $|S|=k$ makes $foundSO$ \textsc{True}} {$k\leftarrow k+1$}
		
	}
	\Return \textsc{Failure}
\end{algorithm}

IPE is, of course, by no means a GA. Nevertheless, their working principles are similar in many aspects. First, they both adopt a population. Second, \textsc{TestSO} in IPE imitates the survival-of-the-fittest behavior of GAs in an extreme manner. IPE adopts the mechanism of partial enumeration, which does not appear in most GAs. However, as long as the size of partial enumeration is bounded such that all patterns occurring in the partial enumeration exist in a sufficiently large population with a high probability, GAs can perform similar recombination without partial enumeration. This is the key argument in the supply theory~\cite{Goldberg:supply:2001} in GA. Finally, IPE invokes \textsc{TestSO} on sets of loci in ascending order by size. This corresponds to the fact that shorter fragments tend to converge faster than longer ones in GAs due to a more sufficient initial supply and fewer competitors in general~\cite{Goldberg:doi:2002}.

In IPE, whenever \textsc{TestSO} returns \textsc{False}, $U$ remains unchanged. For ease of analysis, we define several terms as follows based on the $i$-th time that \textsc{TestSO} returns \textsc{True}.

\begin{definition}\text{ }
	\begin{itemize}
	\item Let \textbf{$S_i$} be the set $S$ at the $i$-th time that \textsc{TestSO} returns \textsc{True}.
	\item Let \textbf{$U_i$} be the set $U$ \textit{immediately after} the $i$-th time that \textsc{TestSO} returns \textsc{True} in Algorithm~\ref{alg:ie} ($U_0 = V$ and $S_i \subseteq U_i$).
	\item Let $\mathcal{T}_i(v) = \mathcal{C}(\mathcal{M}_{SO}(v)) - (V-U_i)$, which is the set of remaining unassigned loci of the MSO.
	\end{itemize}
\end{definition}

Now we start analyzing the size of $S_i$. The following proposition states that the size of $S_i$ is upper bounded by $\min_{v\in U_i} |\mathcal{T}_i(v)|$.

\begin{proposition}
	If $V-U_i$ is correctly assigned, $|S_i| \leq \min_{v\in U_i} |\mathcal{T}_i(v)|$.\label{prop:pmso}
	\begin{proof}
	The statement holds because \textsc{TestSO} returns \textsc{True} on the input $S=\mathcal{T}_i(v)$ for all $v\in U_i$, given $V-U_i$ is correctly assigned (Proposition~\ref{prop:so}), and IPE tests $S$ in ascending order by size.
	\end{proof}
\end{proposition}

To practically upper bound the size of $S_i$, next we show that $\min_v |\mathcal{T}_i(v)|$ is upper bounded when the size of SCCs and the maximum in-degree are upper bounded. In the rest of this paper, we often need to refer to the above condition to ensure decomposability, and hence we give the following definition.

\begin{definition}
The \textbf{decomposition difficulty} of the problem (or EG) is $\max(k_{SCC}, k_{in}+1)$, where
\begin{itemize}
	\item the maximum size of SCCs of its EG is $k_{SCC}$, and 
	\item the maximum in-degree of its EG is $k_{in}$.
\end{itemize}\label{def:diff}
\end{definition}

Note that the minimum decomposition difficulty is $1$, which corresponds to the case of \textsc{OneMax}. We deliberately add one to the maximum in-degree to conceptually map the decomposition difficulty to the size of building blocks.

\begin{proposition}	
	For any EG with decomposition difficulty of $k$, we have $\min_v |\mathcal{IN}^*(v)| \leq k$.\label{prop:in-closure-is-k}
	\begin{proof}
	Consider an arbitrary vertex $v$ and its in-closure $\mathcal{IN}^*(v)$. Either of the following conditions holds.
	\begin{enumerate}	
	\item $\exists u \in \mathcal{IN}^*(v)$, $v \not\in \mathcal{IN}^*(u)$.
	\item $\forall u\in \mathcal{IN}^*(v)$, $v \in \mathcal{IN}^*(u)$.
	\end{enumerate}
	For any vertex $v$ that belongs to case 1, we have $\mathcal{IN}^*(u) \subset \mathcal{IN}^*(v)$. The minimum in-closure does not occur at vertex $v$; instead, it only occurs at some vertex that belongs to case 2. For any $v$ in case 2, we know that $\mathcal{IN}^*(u) = \mathcal{IN}^*(v)$ and that $v$ and $u$ belong to the same SCC, which is exactly their in-closures. Since the maximum size of SCCs is less than or equal to $k$, the statement holds.
	\end{proof}
\end{proposition}

Finally, we have $|S_i|$ upper bounded by a constant under the following conditions.

\begin{lemma} 
In IPE, if the decomposition difficulty of the problem is $k$ and $V-U_i$ is correctly assigned, we have $|S_i| \leq k$.\label{lemma:s-k}
\begin{proof}
Consider the subgraph of the EG defined on $U_i$, namely $G_i$. The decomposability of $G_i$ is less than or equal to $k$. Note that $\mathcal{T}_i(v) = \mathcal{C}(\mathcal{M}_{SO}(v)) - (V-U_i)$ is the coverage of MSO of the subproblem defined on $U_i$. The minimum size, $\min_{v \in U_i} |\mathcal{T}_i(v)|$, is less than or equal to $k$ according to Proposition~\ref{prop:in-closure-is-k}. We then have $|S_i| \leq k$ directly by Proposition~\ref{prop:pmso}.
\end{proof}
\end{lemma}

Note that even if $V-U_i$ is correctly assigned, $S_i$ is not necessarily  $\mathcal{T}_i(  \argmin_{v\in U_i} |\mathcal{T}_i(v)|)$. \textsc{TestSO} may return \textsc{True} on $S_i$ smaller than that, but the question is that whether \textsc{TestSO} returns \textsc{True} with the correct assignment, $A=\{(S_i, g)\}$.

We call the next theorem the \emph{epistasis blanket theorem}, which is one of our major results.
It states that given a set of loci, as long as the 1st tier of in-nodes is set correctly, there exists a specific assignment on the 2nd tier of in-nodes that guarantees the correct gene assignments on that given set of loci belong to the constrained optima, no matter what the assignment is on the remaining loci.

\begin{theorem}[Epistasis Blanket Theorem]
	$\forall S \subseteq V, \forall R$, we have $g[s] \in \Psi_{\{( \mathcal{IN}(S)-S, g)\} \cup R}[s]$ for all $s\in S$, where $R$ is an arbitrary assignment with the coverage $\mathcal{C}(R) = V- S - \mathcal{IN}(S) - \mathcal{IN}^2(S)$.\label{thm:blanket}
	\begin{proof}
		Suppose the claim is not true. $\exists R$ with $\mathcal{C}(R) = V- S- \mathcal{IN}(S) - \mathcal{IN}^2(S), \Psi_{\{( \mathcal{IN}(S)-S, g)\} \cup R} [s] = \{\overline{g[s]}\}$ for some $s\in S$.
		Now we keep dropping gene assignments until the assignment on constraint is minimum while the constrained optimal assignment at $s$ remains $\overline{g[s]}$: 
		$\Psi_{\{(S', g)\} \cup R'} [s] = \{\overline{g[s]}\}$, where $S'\subseteq \mathcal{IN}(S)-S$, $R' \subseteq R$, and $S'\cup \mathcal{C}(R')$ is minimum. Note that $S'$ cannot be empty; otherwise, $\Psi_{R'} [s] = \{\overline{g[s]}\}$ indicates that $\exists r\in \mathcal{C}(R'), \{r\}\Rightarrow s$ (Proposition~\ref{prop:someone}), which makes $r \in \mathcal{IN}(S)$. Also, $R'$ cannot be empty since we know that $\Psi_{\{(S', g)\}}[s] = \{g[s]\}$. 
		
		Since the size of such an assignment is minimum, and both $S'$ and $R'$ are nonempty, $\Psi_{\{(S', g)\} \cup R'} \neq \Psi_{R'}$. By Proposition~\ref{prop:alter}, 
		$\exists s'\in S',\; \overline{g[s']} \in \Psi_{R'}[s']$. By Proposition~\ref{prop:someone}, $\exists r \in \mathcal{C}(R')$ such that $\{r\} \Rightarrow s'$. However, we have $r\not\in \mathcal{IN}^2(S)$ and $s'\in \mathcal{IN}(S)$, which leads to a contradiction.
	\end{proof}
\end{theorem}

Note that if $\mathcal{IN}(S)=S$, we have $\mathcal{IN}^*(S)=S$. In this case, the blanket theorem is reduced to Lemma~\ref{lemma:in-so} (an in-closure is stationarily optimal): $S$ is an in-closure, and no matter what the remaining assignment is on $V-S$, the assignment $\{(S,g)\}$ remains optimal. Figure~\ref{fig:blanket} illustrates the concept of epistasis blanket theorem.

\begin{figure}
\centering
\includegraphics[width=0.7\textwidth]{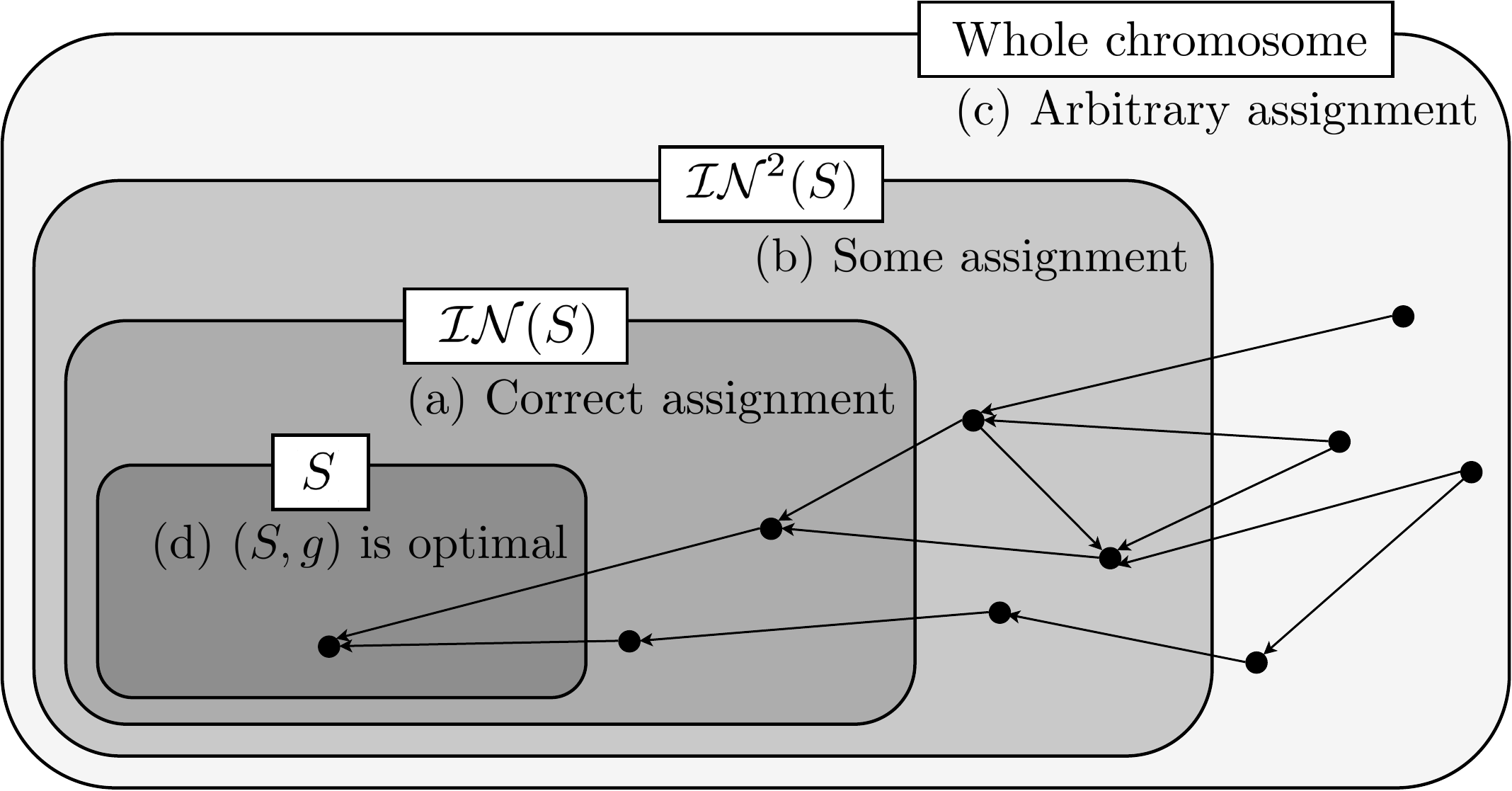}
\caption{Illustration of the epistatic blanket theorem. For any $S$, if (a) $\mathcal{IN}(S)-S$ is correctly assigned, there exists (b) some assignment to $\mathcal{IN}^2(S)-\mathcal{IN}(S)$ such that (c) for any assignment to the rest of the chromosome, (d) $S$ can be correctly assigned according to fitness superiority.}
\label{fig:blanket}

\end{figure}

Under the condition specified in the blanket theorem, for all $s \in S$, at least one constrained optimum contains $g[s]$ at locus $s$. These $g[s]$ may distribute over different constrained optima for different $s$. However, when only one optimal pattern exists on $S$ with the constraint, it must contain all $g[s]$ for all $s$. 

\begin{corollary}
$\forall S \subseteq V, \forall R$, $|\Psi_{\{( \mathcal{IN}(S)-S, g)\} \cup R}[s]| = 1$ for all $s\in S$ implies that $\psi[s] = \{g[s]\}$ for all $s\in S$ and for all $\psi \in \Psi_{\{( \mathcal{IN}(S)-S, g)\} \cup R}$, where $R$ is an arbitrary assignment with the coverage $\mathcal{C}(R) = V- S - \mathcal{IN}(S) - \mathcal{IN}^2(S)$.
\label{coro:blanket}
\end{corollary}

With the epistasis blanket theorem and its corollary, we can upper bound the probability of \textsc{TestSO} being mistaken given the decomposition difficulty is $k$.

\begin{lemma}
	Suppose the decomposition difficulty is $k$. In IPE, the probability of \textsc{TestSO} returning \textsc{True} but $S_i$ is not correctly assigned, \textit{i.e.}, the returned assignment is not $\{(S_i,g)\}$, is less than or equal to $e^{-n/2^{k^2+k^3}}$. \label{lemma:prob-bound-1}
	\begin{proof}
	First, we consider any time when \textsc{TestSO} is invoked with either \textsc{True} or \textsc{False} returned. By Theorem~\ref{thm:blanket}, as long as $\mathcal{IN}(S)-S$ is set to $g$, there exists an assignment $A$ on $\mathcal{IN}^2(S)-\mathcal{IN}(S)-S$ such that for all $s\in S$, at least one constrained optimum has $g[s]$ at $s$ no matter what the remaining assignment $R$ is on $V- S- \mathcal{IN}(S) - \mathcal{IN}^2(S)$. 
	
	Under the above condition, $|\Psi_{\{( \mathcal{IN}(S)-S, g)\} \cup R}[s]| > 1$ for any $s\in S$ causes \textsc{TestSO} to return \textsc{False}. For \textsc{TestSO} to return \textsc{True}, $|\Psi_{\{( \mathcal{IN}(S)-S, g)\} \cup R}[s]|$ must be 1 for all $s \in S$, and the constrained optimal pattern on $S$ is $\{(S,g)\}$ (Corollary~\ref{coro:blanket}).
	
	When \textsc{TestSO} returns \textsc{True}, the input $S$ is defined as $S_i$. For \textsc{TestSO} to return \textsc{True} but with incorrect assignment on $S_i$, the assignment $\{(\mathcal{IN}(S_i)-S_i, g)\} \cup A$ must not exist in the population. An upper bound of such a probability can be calculated since we know $|S_i| \leq k$ (Lemma~\ref{lemma:s-k}), in-degree $< k$, $|\{(\mathcal{IN}(S_i)-S_i, g)\} \cup A| \leq k^2+k^3$, and that the chromosomes are uniformly randomly initialized:
\begin{linenomath*}
\begin{equation*}
Pr \leq \left(1- \frac{1}{2^{k^2+k^3}} \right)^n.
\end{equation*}
\end{linenomath*}	
With $\left(1-\frac{x}{n}\right)^n \leq e^{-x}$, 
$Pr \leq e^{-n/2^{k^2+k^3}}.$
\end{proof}
\end{lemma}

Furthermore, the probability of IPE failing to find the global optimum is upper bounded as follows.

\begin{lemma}
If the decomposition difficulty of the problem is $k$, IPE returns the global optimum with probability greater than or equal to $\left( 1- e^{-n/2^{k^2+k^3}} \right)^{\ell}.$
\label{lemma:prob-bound-2}
\end{lemma}
\begin{proof}
Whenever \textsc{TestSO} returns \textsc{True}, IPE assigns and fixes the alleles on at least one locus. Therefore, at most $\ell$ invocations of \textsc{TestSO} return \textsc{True}.\\
Let $\alpha_i$ be the event of the $i$-th time that \textsc{TestSO} returns \textsc{True}. Let $\beta_i$ be the event of the returned assignment being correct at the $i$-th time that \textsc{TestSO} returns \textsc{True}.\\
The probability of IPE returning the global optimum is then $Pr\left( \bigwedge_i (\beta_i \mid \alpha_i) \right)$.  Note that these events occur in a sequential manner. Once $\beta_i$ occurs, the corresponding alleles are correctly assigned and fixed and hence do not affect the remaining alleles in any way.  We can think of those corresponding vertices as pruned from the EG and do not affect the remaining graph anymore. As a result, these events become independent chronologically, and the probability is equivalent to  $\prod_i Pr\left( \beta_i \mid \alpha_i\right)$. We already have the bound of $1-Pr\left( \beta_i \mid \alpha_i\right)$ in Lemma~\ref{lemma:prob-bound-1}, which leads to the conclusion of the statement.
\end{proof}

With all the analyses above, now we are ready to link the whole thing with PAC learning.

\begin{theorem} If the decomposition difficulty of the problem is $O(1)$, IPE returns the global optimum with probability $(1-\delta)$ within the number of function evaluations polynomial to $\ell$ and $\frac{1}{\delta}$ with an appropriate $n$.
\label{thm:pac}
\begin{proof}		
Let the decomposition difficulty be $k$. By Lemma~\ref{lemma:prob-bound-2}, the probability that IPE returns the global optimum is greater than or equal to $\left( 1-e^{-n/2^{k^2+k^3}} \right)^{\ell}$.
		
For $0<p<1$, $(1-p)^{\ell} \geq 1-{\ell} p$, and  we desire the probability greater than or equal to a threshold.
\begin{linenomath*}
\begin{equation*}
\left( 1-e^{-n/2^{k^2+k^3}} \right)^{\ell} \geq 1-{\ell} \cdot e^{-n/2^{k^2+k^3}} \geq 1-\delta.	
\end{equation*}
\end{linenomath*}
Solving the latter inequality yields
\begin{linenomath*}
\begin{equation*}
n \geq 2^{k^2+k^3} \cdot (\ln \ell + \ln \frac{1}{\delta}).
\end{equation*}
\end{linenomath*}

The above $n$ is sufficiently large for IPE to return the global optimum with probability greater than or equal to $(1-\delta)$. The total number of function evaluations can be calculated as follows. The outer loop terminates when $U$ becomes empty, and each time that \textsc{TestSO} returns \textsc{True}, $|U|$ is decremented by at least 1. Let $|S_i| = k_i$. The number of the invocations of \textsc{TestSO} is at most $(C^\ell_1+C^\ell_2+\ldots+C^\ell_{k_0}) + (C^\ell_1+C^\ell_2+\ldots+C^\ell_{k_2}) + \ldots +(C^\ell_1+C^\ell_2+\ldots+C^\ell_{k_t})$, where $t < \ell$ is the number of outer iterations. Since all $k_i\leq k$ (Lemma~\ref{lemma:s-k}), the number of the invocations of \textsc{TestSO} is $O(\ell^{k+1})$. 
		
Each invocation costs $O(n 2^k)$ function evaluations, which makes the total number of function evaluations 
\begin{linenomath*}
\begin{equation*}	
O\left(2^{k+k^2+k^3} \ell^{k+1} \cdot (\ln \ell + \ln \frac{1}{\delta})\right).
\end{equation*}
\end{linenomath*}
Since $k$ is a constant, this is equivalent to 
\begin{linenomath*}
\begin{equation*}	
O\left(\ell^{k+1} \cdot (\ln \ell + \ln \frac{1}{\delta})\right),
\end{equation*}
\end{linenomath*}
which is polynomial to $\ell$ and $\frac{1}{\delta}$.
\end{proof}
	
\end{theorem}

Finally, with the definition of PAC for optimization in Section~\ref{sec:pac}, we have the following corollary, which is a special case where $\epsilon = 0$. 

\begin{corollary}
The global optima are PAC learnable for problems with decomposition difficulty of $O(1)$.
\end{corollary}

\noindent MBGAs build models to decompose the problems and then recombine chromosomes with promising partitions, possibly MSOs or SOs, to construct the optimum. In this section, we show that under certain conditions, SOs are PAC decidable.

\subsection{SOs of Bounded Sizes Are PAC Decidable Under Sufficiently Large Minimum Deception Rate}

If we are only interested in deciding whether a given set is the coverage of an SO, we can discard the part of partial enumeration and merely invoke \textsc{TestSO} on that set. The concept is decidability that follows the PAC framework, and hence we define PAC decidability as follows.

\begin{definition}
\label{def:PAC-dec}
A concept is \textbf{PAC decidable} if it can be correctly decided (yes or no) with probability $(1-\delta)$ in time polynomial to $\frac{1}{\delta}$ and the problem size.
\end{definition}

For problems where the global optimum can be found without correctly identifying SOs, the identification of SOs is unnecessary because there are insufficient misleading assignments. Accordingly, we focus on those problems for which correctly identifying SOs is necessary to achieve the global optimum. That is, there exist sufficiently many misleading assignments such that without correctly identifying SOs, the global optimum cannot be found. We define the concept of deception rate as follows.

The significance of Theorem~\ref{thm:pac} and its following corollary lies in the scaling behavior of the population size $n$ relative to the problem's structural properties. By establishing that the probability of an incorrect decision by \textsc{TestSO} decays exponentially with $n$, we demonstrate that achieving high confidence does not necessitate an infinitely large population. Rather, the reliability of SO identification is governed by the interplay between the decomposition difficulty $k_d$ and the minimum deception rate $r$.

\begin{definition}
The \textbf{minimum deception rate} $r$ of a problem with EG $= (V,E)$ is defined as 

\begin{tabular}{p{4cm}p{1cm}p{4cm}}
$r =
\begin{cases} 
	\displaystyle \min_{e \in E}\; r_{e} & \text{ if } E \neq \phi\\ 
	\hfil 0 & \text{ if } E = \phi 
\end{cases}$	&
\text{   , where   } &
$
r_e = \displaystyle \min_{\substack{S}} \frac{|\mathcal{D}|}{|\mathcal{A}|},
$
\\
\end{tabular}
\\
$e=(u,v)$, $u \in S$, $v \not\in S$, $\mathcal{A} = \{ A \mid \mathcal{C}(A) = S\}$ and $\mathcal{D} = \{ A \mid \mathcal{C}(A) = S$ and $\overline{g[v]} \in \Psi_A[v]\}$.
\end{definition}

In the above definition, $\mathcal{A}$ is the set of all assignments with coverage of $S$, and $\mathcal{D}$ is the set of all deceptive assignments with coverage of $S$, and $v$ being incorrectly assigned according to constrained optima. Conceptually, the minimum deception rate provides a lower bound on the probability that, when the other genes are fixed, a uniformly generated chromosome has a gene incorrectly assigned by fitness superiority.
Note that since for all $A$, $\Psi_{\{(\mathcal{C}(A), g)\}} = g$ and hence  $\overline{g[v]} \notin \Psi_{\{(\mathcal{C}(A), g)\}}[v]$. It follows that $r \in [0, 1)$.

Below are some examples of the minimum deception rate. 
In \textsc{OneMax}, all genes are independent, \textit{i.e.}, $E = \phi$, so the minimum deception rate is zero. In the 4-bit \textsc{CTrap}, which consists of one trap defined over $V = \{0, 1, 2, 3\}$, all genes are epistatically related to each others and hence $(u, v) \in E$ for all $u, v \in V$. In this problem, $r_{e} = 0.5$ for all $e$, resulted in $r=0.5$. As an example, the calculation of $r_{(2, 3)} = 0.5$ is detailed in Table~\ref{tbl:deception}. Consider any set $S$ where $2 \in S$ and $3 \not\in S$. $\mathcal{A}$ contains all possible assignments over $S$ and hence $|\mathcal{A}| = 2^{|S|}$. Among these assignments, the deceptive condition 
$\overline{g[v]} \in \Psi_A[v]$ occurs on all but one assignment: ${\{(S, g)\}}$ and hence $|\mathcal{D}| = 2^{|S|}-1$. Since the minimum value of $|\mathcal{D}|/|\mathcal{A}|$ for any given $S$ is 0.5, $r_{(2, 3)} = 0.5$. 

The minimum deception rate characterizes the likelihood that deceptive assignments induce incorrect fitness superiority when the remaining genes are fixed. Consequently, it determines the lower bound of probability that \textsc{TestSO} is misled when evaluating whether or not a candidate set satisfies the SO condition on a uniformly generated chromosome. The following lemma and theorem state that SOs are PAC decidable under the condition that the minimum deception rate is bounded by a constant.

\begin{table}[t]
\centering
\caption{In the 4-bit \textsc{CTrap} problem, $r_{(2, 3)}$ is $0.5$. Each row corresponds to a subset $S$ such that $2 \in S$ and $3 \notin S$, and shows the values of $|\mathcal{A}|$, $|\mathcal{D}|$, and the ratio $|B'|/|B|$. The minimum ratio, $0.5$, determines $r_{(2, 3)}$.}
\label{tbl:deception}

\begin{tabular}{p{3.5cm}cp{3.5cm}}
&
\begin{tabular}{lccl}
\toprule
$S$ & $|\mathcal{D}|$ & $|\mathcal{A}|$ & $|\mathcal{D}|/|\mathcal{A}|$\\ \hline
$\{2\}$ & 1 & 2 & 0.5  \\
$\{0, 2\}$ & 3 & 4 & 0.75 \\
$\{1, 2\}$ & 3 & 4 & 0.75 \\
$\{0, 1, 2\}$ & 7 & 8 & 0.875  \\
\hline
\end{tabular}
& \\
\end{tabular}
\end{table}

\begin{lemma}
	Suppose that (1) the decomposition difficulty is $k_d$, and (2) the minimum deception rate is $r$, the probability of \textsc{TestSO} returning \textsc{True} on the input $S$ while $S$ is not an SO is less than or equal to $e^{-n / \max(2^{k_d^2+k_d^3}, \frac{1}{r})}$.
	\label{lemma:mdr}
	\begin{proof}
		If $S$ contains (if not is) only part of an in-closure while \textsc{TestSO} returns \textsc{True}, there exist $u \not\in S$ and $v\in S$ such that $\{u\} \Rightarrow v$. The blanket theorem (Theorem 21) states that as long as $\mathcal{IN}(S)-S$ is set correctly, there exists an assignment $A$ on $\mathcal{IN}^2(S)-\mathcal{IN}(S)-S$ that makes the constrained optimal pattern ${g[v]}$ at locus $v$. On the other hand, an assignment \(A_r\) on $V-S$ is one that causes the optimal assignment at \(v\) to differ from \(g[v]\). As long as these two assignments, $\{(\mathcal{IN}(S)-S, g)\} \cup A$ and $A_r$, exist in the population, \textsc{TestSO} returns \textsc{False}.

Let event \(\alpha\) be the event that a randomly generated chromosome contains the assignment of the blanket. Since the decomposition difficulty is \(k_d\), we have
\[
\Pr(\alpha) \geq \frac{1}{2^{k_d^2+k_d^3}}.
\]

Let event \(\beta\) be the event that a randomly generated chromosome contains the assignment \(A_r\). Since the minimum deception rate is \(r\), we have
\[
\Pr(\beta) \geq r.
\] Hence, the probabilities of events \(\alpha\) and \(\beta\) are lower bounded by constants greater than zero.

Let event \(\alpha'\) be the event that half of the population does not contain the assignment of the blanket, and let event \(\beta'\) be the event that the other half of the population does not contain the assignment corresponding to $A_r$. Finally, let event $\gamma$ be the event of not having both assignments in the whole population. Since $\gamma$ entails ($\alpha'$ and $\beta'$), we know that the probability of $\gamma$ is less then or equal to the joint probability of $\alpha'$ and $\beta'$. In addition, we know that events $\alpha'$ and $\beta'$ are independent from the semantic. As a result, the following inequalities hold.

\begin{linenomath*}
    \begin{eqnarray*}
        Pr(\gamma) &\leq& Pr(\alpha' \wedge \beta') = Pr(\alpha')\cdot Pr(\beta')\\
        &\leq& \left(1-\frac{1}{2^{k_d^2+k_d^3}}\right)^{n/2} \cdot (1-r)^{n/2}\\
        &\leq& e^{-n / 2^{1+k_d^2+k_d^3}} \cdot e^{ -n / 2^{1+\log_2(\frac{1}{r})}}\\
        &\leq& e^{-n / \max(2^{k_d^2+k_d^3}, \frac{1}{r})}. 
    \end{eqnarray*}
\end{linenomath*}
\end{proof} 
\end{lemma}

\begin{theorem} If the decomposition difficulty of the problem is $k_d$ and the minimum deception rate of the problem is $r$, \textsc{TestSO} returns the correct SO with probability $(1-\delta)$ within the number of function evaluations polynomial to $\frac{1}{\delta}$ with an appropriate $n$.
\label{thm:pac}
\begin{proof}		
By Lemma~\ref{lemma:mdr}, the probability that \textsc{TestSO} returns an SO is greater than or equal to $1-e^{-n / \max(2^{k_d^2+k_d^3}, \frac{1}{r})}$. We desire the probability greater than or equal to a threshold.
\begin{linenomath*}
\begin{equation*}
1-e^{-n/\max(2^{k_d^2+k_d^3}, \frac{1}{r})} \geq 1-\delta.	
\end{equation*}
\end{linenomath*}
Solving the latter inequality yields
\begin{linenomath*}
\begin{equation*}
n \geq {2^{\max(k_d^2+k_d^3, \frac{1}{r})} \cdot \ln \frac{1}{\delta}}.
\end{equation*}
\end{linenomath*}

The above value of $n$ is sufficiently large for \textsc{TestSO} to return the correct SO with probability at least $(1 - \delta)$. Specifically, a population of size $2^{\max(k_d^2 + k_d^3, \frac{1}{r})} \cdot \ln \frac{1}{\delta}$ is used in \textsc{TestSO}, which results in a total of 
\[
O\left(2^{|S|} \cdot 2^{\max(k_d^2 + k_d^3, \frac{1}{r})} \cdot \ln \frac{1}{\delta}\right)
\]
function evaluations, which is polynomial to $\frac{1}{\delta}$, where $S$ is the input to \textsc{TestSO}. 

\end{proof}
	
\end{theorem}

With the definition of PAC decidability (Definition~\ref{def:PAC-dec}), we have the following corollary.

\begin{corollary}
SOs are PAC decidable if decomposition difficulty is $O(\log^{1/3} \ell)$ and the minimum deception rate of epistasis is $\Omega(\frac{1}{\log \ell})$, and the maximum size of SOs is $O(\log \ell)$.
\end{corollary}

\subsection{Maximum In-degree Upper Bounds The Decomposition Difficulty on EGs with Only Strict Epistases}
\label{sec:cyclic-eg-with-bounded-in-degree}
Recall that the decomposition difficulty is defined according to both the maximum in-degree and the maximum size of SCCs (Definition~\ref{def:diff}). However, if the EG contains only strict epistases (Assumption~\ref{asm:strict}), bounded in-degree implies boundedly sized SCCs. We start the derivations by first defining epistatic equivalency.

\begin{assumption}
	The EG of the problem contains only strict epistases. 
	\label{asm:strict}
\end{assumption}

\begin{definition}
	$u$ and $v$ are \textbf{epistatically equivalent} if $u\rightarrow v$ and $v \rightarrow u$.
\end{definition}

If two vertices are epistatically equivalent, the strictly epistatic relation propagates as follows.
\begin{lemma}
	\label{lemma:equivalent}
	Given that $u$ and $v$ are epistatically equivalent, $u\rightarrow w$ implies $v\rightarrow w$.
	\begin{proof}
		Since $u\rightarrow v$, by definition $\Psi_{\{(u, \overline{g})\}}[v] = \{\overline{g[v]}\}$. By Proposition~\ref{prop:opt-remain1}, we have $\Psi_{\{(u, \overline{g})\}} = \Psi_{\{(u, \overline{g}), (v, \overline{g})\}}$.
		Likewise, since $v\rightarrow u$, we have $\Psi_{\{(v, \overline{g})\}} = \Psi_{\{(u, \overline{g}), (v, \overline{g})\}}$.
		As a result, $\Psi_{\{(v, \overline{g})\}}[w] = \Psi_{\{(u, \overline{g})\}}[w]$.\\
		Since $u\rightarrow w$, by definition $\Psi_{\{(u, \overline{g})\}}[w] = \{\overline{g[w]}\}$. From above, we also have $\Psi_{\{(v, \overline{g})\}}[w] = \{\overline{g[w]}\}$, and hence $v \rightarrow w$.
	\end{proof}
\end{lemma}

With these tools, we show that cycles only occur in cliques and that maximal cliques do not share any common vertex. The derivations start by showing that there are no chordless cycles except those of size 2, in which case $u\rightarrow v$ and $v\rightarrow u$. Because chordless cycles are not well defined in directed graphs, here we give mathematical expressions of what we mean by that. 

\begin{definition}
Consider the following cycle $v_0 \rightarrow v_1 \rightarrow \cdots \rightarrow v_{c-1} \rightarrow v_0$, where $\forall i$, $v_{i} \in V$ and $c \leq \ell$ is the size of the cycle. For convenience, define $a\oplus b$ as the result of $(a+b)$ modulo $c$, and $a\ominus b$ as the result of $(a-b)$ modulo $c$. We call the cycle \textbf{chordless} if $\forall i, \forall j\not=i\oplus1$, $v_i \not\rightarrow v_j$.\footnote{By this definition, in the cycle $a\rightarrow b\rightarrow c\rightarrow a$, edge $b\rightarrow a$ is considered a chord.}
\end{definition}
 
\begin{lemma}
\label{lemma:chordless}
There is no chordless cycle of size greater than 2 in the EG. 
\begin{proof}
Suppose that there exists a chordless cycle of size greater than 2. There exist three adjacent vertices on the cycle, namely $v_{i\ominus1}$, $v_{i}$, and $v_{i\oplus 1}$. Let $(a, b, c)$ denote the assignment where the alleles of $v_{i\ominus1}$, $v_{i}$, and $v_{i\oplus 1}$ are set to $a$, $b$, and $c$, respectively.
		
	Since $v_{i\ominus1} \rightarrow v_i$ and $v_{i\ominus1} \not\rightarrow v_{i\oplus1}$,   $( \overline{g[v_{i\ominus1}]}, \overline{g[v_{i}]}, g[v_{i\oplus1}] )$ occurs in $\Psi_{\{(v_{i\ominus1}, \overline{g})\}}$.
	
	Since $v_{i} \rightarrow v_{i\oplus1}$ and $v_{i} \not\rightarrow v_{i\ominus1}$,  $( g[v_{i\ominus1}], \overline{g[v_{i}]}, \overline{g[v_{i\oplus1}]} )$ occurs in $\Psi_{\{(v_{i}, \overline{g})\}}$.

	We can conclude that $f(\{( \{v_i, v_{i\oplus 1}\}, \overline{g}) \} ) > f( \{( \{v_{i\ominus 1}, v_i\}, \overline{g}) \} ) $; otherwise, we would have $v_i \rightarrow v_{i\ominus 1}$. Since $i$ is arbitrary and the relation is cyclic, eventually we have $f(\{( \{v_i, v_{i\oplus 1}\}, \overline{g}) \} ) > f(\{( \{v_i, v_{i\oplus 1}\}, \overline{g}) \} )$, which leads to a contradiction.
	\end{proof}
\end{lemma}

Next, we show that all cycles occur only in cliques.

\begin{lemma}
	Cycles only occur in cliques.
	\label{thm:clique}
\end{lemma}
\begin{proof}
	Consider the cycle $v_0 \rightarrow v_1 \rightarrow \cdots \rightarrow v_{c-1} \rightarrow v_0$. Prove by induction on $c$.
	\begin{enumerate}
		\item Base: 
		\begin{itemize}
			\item $c=1$. There is no cycle.
			\item $c=2$. The cycle itself is a clique.
			\item $c=3$. By Lemma~\ref{lemma:chordless}, there exists a chord. Without loss of generality, let the chord be $v_1 \rightarrow v_0$. Since $v_0$ and $v_1$ are epistatically equivalent, $v_1 \rightarrow v_2$ implies $v_0 \rightarrow v_2$ by Lemma~\ref{lemma:equivalent}. Now that $v_0$ and $v_2$ are epistatically equivalent, $v_0 \rightarrow v_1$ implies $v_2 \rightarrow v_1$, making $v_0$, $v_1$ and $v_2$ a clique.
		\end{itemize}
		\item Induction hypothesis:
		The statement holds for $c \leq i$.
		\item Consider $c=i+1$. By Lemma~\ref{lemma:chordless}, there exists a chord. Without loss of generality, let the chord be $v_z \rightarrow v_0$, where $1\leq z \leq i-1$.
		
		First, we claim that the whole cycle contains two splitting cliques sharing $v_0$, namely $C_1$ and $C_2$, as follows.
		
		\begin{itemize}
			\item $z=1$.\\
			$v_0$ and $v_1$ are epistatically equivalent. $v_1 \rightarrow v_2$ implies $v_0 \rightarrow v_2$. $C_1 = v_0 \rightarrow v_2 \rightarrow v_3 \rightarrow \ldots \rightarrow v_i \rightarrow v_0$ is a cycle of size $i$, which is a clique by the induction hypothesis. Therefore, we have $v_2 \rightarrow v_0$ making $C_2 = v_0 \rightarrow v_1 \rightarrow v_2 \rightarrow v_0$ a cycle of size 3, which is the base case, and hence $C_2$ is also a clique.
			\item $z>1$.\\
			Let $C_1$ denote $v_0 \rightarrow v_1 \rightarrow \ldots \rightarrow v_z \rightarrow v_0$. $C_1$ is a cycle of size $z+1$. Since $z+1 \leq i$, by the induction hypothesis, $C_1$ is a clique, and hence $v_0 \rightarrow v_z$.\\
			Now consider $C_2 = v_0 \rightarrow v_z \rightarrow v_{z+1} \rightarrow \ldots \rightarrow v_{i} \rightarrow v_0$. $C_2$ is a cycle of size $i-z+2$. Again, $i-z+2 \leq i$ due to $z>1$, and hence $C_2$ is also a clique by the induction hypothesis. 
		\end{itemize}
		Now consider any two vertices $v_a$ and $v_b$ in the whole cycle. If both $v_a$ and $v_b$ belong to the same clique, we have $v_a \rightarrow v_b$ and $v_b \rightarrow v_a$.
		If $v_a$ and $v_b$ belong to different cliques, let $v_a \in C_1$ and $v_b \in C_2$ without loss of generality. Since $C_1$ is a clique, $v_0$ is either $v_a$ or is epistatically equivalent to $v_a$. Since $C_2$ is a clique, $v_0 \rightarrow v_b$. By Lemma~\ref{lemma:equivalent}, $v_a \rightarrow v_b$. Symmetrically, $v_b \rightarrow v_a$. Since $a$ and $b$ are arbitrary, $v_0, v_1, \ldots, v_i$ form a clique of size $i+1$. 
	\end{enumerate}
	By the principle of mathematical induction, the statement holds for all sizes of cycles.
\end{proof}

\begin{corollary}
	None of the maximal cliques share common vertices.
	\label{coro:nonoverlapping}
	\begin{proof}
		Suppose two maximal cliques $C_1$ and $C_2$ share a common vertex $v$. Consider another vertex $u\neq v$ in $C_2$. Since every vertex in $C_1$ is epistatically equivalent to $v$, $v \rightarrow u$ implies that every vertex in $C_1$ is strictly epistatic to $u$ (Lemma~\ref{lemma:equivalent}). Symmetrically, $u$ is epistatically equivalent to $v$, which makes $u$ strictly epistatic to every vertex in $C_1$. As a result, $C_1$ and $u$ form a clique which makes $C_1$ non-maximal, and this leads to a contradiction.
	\end{proof}
\end{corollary}

If a problem consists of only non-weak strict epistases, the cycles in SCCs occur only in cliques. That makes SCCs maximal cliques, and they do not share common vertices. In other words, such a problem embodies a quality to that of additively separable problems in terms of optimization, even if its fitness function is defined on overlapping subfunctions. 

\begin{theorem}
If the maximum in-degree is $k$ of an EG that contains only non-weak strict epistases, the maximum size of SCCs is $k+1$, and consequently, the decomposition difficulty is $k+1$.
\begin{proof}
Since cycles occur only in cliques, all SCCs are maximal cliques, and hence the statement holds. 
\end{proof}
\end{theorem}

Non-strict epistases occur only when there exist different chromosomes with equal fitness value(s), which is often the case for combinatorial problems. For non-combinatorial problems, such as the NK-landscape, the fitness is real-valued, and the values are rarely equal. On those problems with only strict epistases, bounded maximum in-degree implies bounded decomposition difficulty, and the aforementioned PAC learnable properties apply.

\section{Conclusion}
\label{sec:conc}

\noindent In this paper, we gave a formal definition of epistasis to represent the concept of linkage. With the definition, we proved several theorems showing the working principles of MBGAs. Specifically, we would like to highlight two important theorems among others developed in this paper. The problem decomposition theorem (Theorem~\ref{thm:m=in}) formally relates the problem decomposition in GA to the concept of linkage by stating that the minimum stationary optima are the in-closures in the epistatic graph of the problem. The epistasis blanket theorem (Theorem~\ref{thm:blanket}) states that as long as a small portion of loci that are epistatic (the 1st and 2nd tiers) to the target loci are set in a particular pattern, the stationarily optimal pattern at the target loci is indeed the global optimal pattern no matter what the remaining assignments are. With these theorems, we presented different results that are related to complexity, PAC learning with different preconditions. Table~\ref{tbl:summary} summarizes our main results. Here, `in polynomial time' assumes polynomial-time fitness evaluation; otherwise, the results apply only to a polynomial number of evaluations.

\begin{table*}
\caption{Summary of major results under different preconditions.}
\label{tbl:summary}
\begin{tabular}{lp{0.8\textwidth}}
Global assumptions: & \parbox{0.79\textwidth}{\begin{itemize}
\item Unique global optimum
\item The problem does not contain weak epistasis
\end{itemize}}\\
\end{tabular}
 
\begin{tabular}{c}
Preconditions:\\
\begin{tabular}{ll}
	(A) EG is known & (B) EG contains only strict epistases\\
	(C) size of SCCs is $O(1)$ & (D) size of SCCs is $O(\log \ell)$\\
	(E) in-degree is $O(1)$ & (F) in-degree is $O(\log \ell)$\\
	(G) size of SOs is $O(\log \ell)$ & (H) size of SCCs is $O(\log^{1/3} \ell)$\\
	(I) minimum deception rate is $\Omega(\frac{1}{\log \ell})$  & (J) in-degree is $O(\log^{1/3} \ell)$\\

\end{tabular}\\
\end{tabular}

\vspace{5px}
\centering
\begin{tabular}{|c|p{5.5cm}|}
	\hline
	Preconditions & \multicolumn{1}{|c|}{Results} \\
	\hline
	\hline
	(A)+(D) or (A)+(B)+(F)  & Global optimum can be obtained in polynomial time. \\
	\hline
	(C)+(E) or (B)+(E) & Global optimum is PAC learnable.\\
	\hline
	(G)+(H)+(I)+(J) or (G)+(B)+(I)+(J) & SOs are PAC decidable.\\
	\hline
\end{tabular}
\end{table*}

 



There is plenty of room in this paper that can be further developed. The concept of linkage can be time-variant, while the epistasis in this paper is not. When a subproblem is nearly converged to the correct alleles, the information of linkage in that subproblem may not be as important for finding the global optimum anymore. Also, most of the results in this paper are obtained by assuming the problem does not contain weak epistasis. We would like to investigate the impact of weak epistases on model quality as well as on finding the global optimum. 

The establishment of theoretical foundations in the GA field is never easy, probably due to the diversity of the stochastic nature and sophisticated characteristics of GAs. While many of those proposed in the literature have been criticized as overly simplified~\cite{Reeves:critique:2002} or only work on additively separable problems~\cite{Eiben:critique:1999}, making progress along this direction of research is still a necessity for the advancement of GAs. Unlike methods in mathematical programming and optimization, such as linear programming and semidefinite programming, for which the scope and specifics of target problems are formally defined and corresponding algorithms are accordingly developed, GAs were proposed based on inspirations from nature and biology and have been proven to work via practice. The scope of applicability for GAs remains unknown, and the lack of rigorous mathematical definitions and fundamentals impedes the theoretical development of GAs. The work presented in this paper provides a formal foundation towards this end by making assumptions only when necessary for the breadth of applicability, while bearing in mind the working principles of MBGAs. Along the line towards understanding the way of problem decomposition, we encountered many concepts that are similar to those in MBGA literature. We cannot say that the epistasis defined in this paper is linkage, or SCCs/in-closures are building blocks, but we do hope that one finds the resemblance between those concepts. In addition, we hope the results in this paper help bridge the GA field with the communities in theoretical machine learning and complexity. 

\pagebreak

\bibliography{submission}

@article{Goldberg:supply:2001,
	author = {Goldberg, David E. and Sastry, Kumara and Latoza, Thomas},
	title = {On the supply of building blocks},
	year = {2001},
	publisher = {Morgan Kaufmann Publishers Inc.},
	address = {San Francisco, CA, USA},
	journal = {Proceedings of the 3rd Annual Conference on Genetic and Evolutionary Computation (GECCO 2001)},
	pages = {336--342},
	location = {San Francisco, California},
	series = {GECCO'01}
}

@article{Etxeberria:enba:1999,
	author = {R. Etxeberria and P. Larra{\~{n}}aga},
	title = {Global optimization using Bayesian networks},
	journal = {Proceedings of the Second Symposium on Artificial Intelligence (CIMAF-99)},
	year = {1999},
	pages ={332--339},
}

@article{Pelikan:hboa:2001,
	author = {Pelikan, Martin and Goldberg, David E.},
	title = {Escaping hierarchical traps with competent genetic algorithms},
	year = {2001},
	publisher = {Morgan Kaufmann Publishers Inc.},
	address = {San Francisco, CA, USA},
	journal = {Proceedings of the 3rd Annual Conference on Genetic and Evolutionary Computation (GECCO 2001)},
	pages = {511--518},
	location = {San Francisco, California},
	series = {GECCO'01}
}

@article{Pelikan:boa:1999,
	author = {Pelikan, Martin and Goldberg, David E. and Cant\'{u}-Paz, Erick},
	title = {{BOA}: the {B}ayesian optimization algorithm},
	year = {1999},
	address = {San Francisco, CA, USA},
	journal = {Proceedings of the 1st Annual Conference on Genetic and Evolutionary Computation (GECCO 1999)},
	pages = {525--532},
	location = {Orlando, Florida},
}

@article{Yu:dsmga:2009,
	author = {Yu, Tian-Li and Goldberg, David E. and Sastry, Kumara and Lima, Claudio F. and Pelikan, Martin},
	title = {Dependency structure matrix, genetic algorithms, and effective recombination},
	year = {2009},
	publisher = {MIT Press},
	address = {Cambridge, MA, USA},
	volume = {17},
	number = {4},
	journal = {Evolutionary Computation},
	pages = {595--626},
}

@article{Eiben:critique:1999,
	title = {Theory of evolutionary algorithms: a bird's eye view},
	journal = {Theoretical Computer Science},
	volume = {229},
	number = {1},
	pages = {3--9},
	year = {1999},
	author = {A.E. Eiben and G. Rudolph},
}

@incollection{Reeves:critique:2002,
	author = {Reeves, Colin R. and Rowe, Jonathan E.},
	booktitle = {Genetic Algorithms: Principles and Perspectives: A Guide to GA Theory},
	title = {Schema Theory},
	chapter = {3},
	year = {2002},
	isbn = {1402072406},
	publisher = {Kluwer Academic Publishers},
	address = {USA},
	pages = {65--93},
}

@techreport{YPC:ll,
	author = {Chen, Y.-p. and Yu, T.-L. and Sastry, K. and Goldberg, D. E.},
	year = {2007},
	title = {A survey of linkage learning techniques in genetic and evolutionary algorithms},
	institution = {IlliGAL},
	number = {2007014},
}

@book{Goldberg:doi:2002,
	author = {Goldberg, David E.},
	title = {The  Design of Innovation: Lessons from and for Competent Genetic Algorithms},
	year = {2002},
	isbn = {1402070985},
	publisher = {Kluwer Academic Publishers},
	address = {USA},
}

@article{Pelikan:decompose:2006,
	author = {Pelikan, Martin and Sastry, Kumara and Butz, Martin V. and Goldberg, David E.},
	title = {Hierarchical {BOA} on random decomposable problems},
	year = {2006},	
	journal = {Proceedings of the 8th Annual Conference on Genetic and Evolutionary Computation},
	pages = {431--432},
}

@incollection{Pelikan:eda:2015,
	author="Pelikan, Martin
	and Hauschild, Mark W.
	and Lobo, Fernando G.",
	editor="Kacprzyk, Janusz
	and Pedrycz, Witold",
	title="Estimation of Distribution Algorithms",
	bookTitle="Springer Handbook of Computational Intelligence",
	year="2015",
	publisher="Springer Berlin Heidelberg",
	address="Berlin, Heidelberg",
	pages="899--928",
}

@article{Pelikan:pairwise:2011,
	author = {Pelikan, Martin and Hauschild, Mark W. and Thierens, Dirk},
	title = {Pairwise and problem-specific distance metrics in the linkage tree genetic algorithm},
	year = {2011},
	isbn = {9781450305570},
	publisher = {Association for Computing Machinery},
	address = {New York, NY, USA},
	journal = {Proceedings of the 13th Annual Conference on Genetic and Evolutionary Computation (GECCO-2011)},
	pages = {1005--1012},
}

@article{Whitley:vig:2016,
	author = {Whitley, Darrell and Chicano, Francisco and Goldman, Brian},
	year = {2016},
	number = {3},
	pages = {491--519},
	title = {Gray box optimization for Mk landscapes ({NK} landscapes and {MAX}-{kSAT})},
	volume = {24},
	journal = {Evolutionary Computation},
	doi = {10.1162/EVCO_a_00184}
}

@article{Valiant:pac:1984,
	author = {L. Valiant},
	journal = {Communications of the ACM}, 
	volume = {27},
	year = {1984},
	title = {A theory of the learnable},
	number = {11},
	pages = {1134--1142},
}

@article{Thierens:OMEA:2011,
	author = {Thierens, Dirk and Bosman, Peter A.N.},
	title = {Optimal mixing evolutionary algorithms},
	year = {2011},
	journal = {Proceedings of the 13th Annual Conference on Genetic and Evolutionary Computation (GECCO 2011)},
	pages = {617--624},
	location = {Dublin, Ireland},	
}

@article{Hsu:dsmga2:2015,
	author = {Hsu, Shih-Huan and Yu, Tian-Li},
	title = {Optimization by Pairwise Linkage Detection, Incremental Linkage Set, and Restricted / Back Mixing: {DSMGA-II}},
	year = {2015},
	journal = {Proceedings of the 2015 Annual Conference on Genetic and Evolutionary Computation (GECCO-2015)},
	pages = {519--526},
}

@article{Doerr:2023:blockleadingones,
	author = {Doerr, Benjamin and Kelley, Andrew James},
	title = {Fourier Analysis Meets Runtime Analysis: Precise Runtimes on Plateaus},
	year = {2023},
	publisher = {Association for Computing Machinery},
	address = {New York, NY, USA},	
	journal = {Proceedings of the {G}enetic and {E}volutionary {C}omputation {C}onference (GECCO-2023)},
	pages = {1555--1564},
}

@article{Rochet:1997:epistasis,
	title = {Epistasis in genetic algorithms revisited},
	journal = {Information Sciences},
	volume = {102},
	number = {1},
	pages = {133--155},
	year = {1997},
	author = {Sophie Rochet},
}

@article{Fan:2013:nfe,
	title = {Linkage learning by number of function evaluations estimation: Practical view of building blocks},
	journal = {Information Sciences},
	volume = {230},
	pages = {162--182},
	year = {2013},
	author = {Kai-Chun Fan and Tian-Li Yu and Jui-Ting Lee},
}

@Inbook{Pelikan:2005:pmbga,
	author= {Pelikan, Martin},
	title= {Probabilistic Model-Building Genetic Algorithms},
	bookTitle={Hierarchical Bayesian Optimization Algorithm: Toward a new Generation of Evolutionary Algorithms},
	year={2005},
	publisher={Springer Berlin Heidelberg},
	address={Berlin, Heidelberg},
	pages={13--30},	
}

@Book{Goldberg:1989:ga-book,
	author = {Goldberg, D. E.},
	year = {1989},
	title = {Genetic Algorithms in Search, Optimization, and Machine Learning},
	address = {Reading, MA},
	publisher = {Addison-Wesley},
}

@book{Larranaga:2002:eda,
	editor = {Larra{\~{n}}aga, P. and Lozano, J.A.},
	year = {2002},
	title = {Estimation of Distribution Algorithms},
	address = {Boston, MA},
	publisher = {Kluwer Academic Publishers}
}

@article{Goldberg:1992:ga-decomposition,
	author =  "Goldberg, D. E. and Deb, K. and Clark, J. H.",
	year =    "1992",
	title =   "Genetic Algorithms, Noise, and the Sizing of Populations",
	journal = "Complex {S}ystems",
	volume =  "6",
	pages =   "333--362"
}

@Book{Holland:1975,
	author =       "Holland, J. H.",
	title =        "Adaptation in natural and artificial systems",
	year =         "1975",
	publisher =    "University of Michigan Press",
	address =      "Ann Arbor, MI",
}

@Incollection{Goldberg:1987:trap,
	author =      {Goldberg, D. E.},
	title =        {Simple genetic algorithms and the minimal, deceptive problem},
	booktitle =    {Genetic Algorithms and Simulated Annealing},
	chapter =      {6},
	year =         {1987},
	pages =        {74--88},
	address =    {London},
	publisher =   {Pitman Publishing}
}

@Article{Pelikan:2003:boa-scalibility,
	author =       "Pelikan, M. and Sastry, K. and Goldberg, D. E.",
	title =        {Scalability of the {B}ayesian optimization algorithm},
	journal =      {International Journal of Approximate Reasoning},
	year =         {2003},
	volume =       {31},
	number =       {3},
	pages =        {221--258},
}

@Article{Yu:2004:model-q-e,
	author = {Tian-Li Yu and David E. Goldberg},
	title = {Quality and efficiency of model building for genetic algorithms},
	journal = {Proceedings of the {G}enetic and {E}volutionary {C}omputation {C}onference (GECCO-2004)},
	year = {2004},
	pages = {367--378},
}

@Article{Yu:2007:pop-sizing,
	author       = {Tian-Li Yu and
	Kumara Sastry and
	David E. Goldberg and
	Martin Pelikan},
	editor       = {Hod Lipson},
	title        = {Population sizing for entropy-based model building in discrete estimation
	of distribution algorithms},
	journal      = {Proceedings of Genetic and Evolutionary Computation Conference, ({GECCO}- 2007)},
	pages        = {601--608},
	publisher    = {{ACM}},
	year         = {2007},
}

@Article{Thierens:2012:evolvability,
	author= {Dirk Thierens and Peter A. N. Bosman},
	title= {Evolvability Analysis of the Linkage Tree Genetic Algorithm},
	journal = {Parallel Problem Solving from Nature (PPSN XII)},
	year = {2012},
	pages= {286--295},
}

@Article{Thirens:2011:om,
	author= {Dirk Thierens and Peter A. N. Bosman},
	title = {The roles of local search, model building and optimal mixing in evolutionary algorithms from a {BBO} perspective},
	journal      = {Proceedings of Genetic and Evolutionary Computation Conference, ({GECCO}- 2011)},
	pages        = {663--670},
	year = {2011},
}

@ARTICLE{Przewozniczek:2020:ell,
	author={Przewozniczek, Michal Witold and Komarnicki, Marcin Michal},
	journal={IEEE Transactions on Evolutionary Computation}, 
	title={Empirical Linkage Learning}, 
	year={2020},
	volume={24},
	number={6},
	pages={1097--1111},
}

@Article{Przewozniczek:2021:dld,
	author = {Przewozniczek, Michal W. and Komarnicki, Marcin M. and Frej, Bartosz},
	title = {Direct linkage discovery with empirical linkage learning},
	year = {2021},	
	journal = {Proceedings of the Genetic and Evolutionary Computation Conference (GECCO-2021)},
	pages = {609--617},	
}

@article{Whitley:2019:gray,
	author = {Whitley, Darrell},
	title = {A gray box manifesto for evolutionary combinatorial optimization},
	year = {2019},
	publisher = {Association for Computing Machinery},
	address = {New York, NY, USA},
	volume = {12},
	number = {1},
	journal = {SIGEVOlution},
	pages = {3--5},
}

\end{document}